\documentclass{siamart251216}

\usepackage{amsmath,amssymb,amsfonts,bm}
\usepackage{mathtools}
\usepackage{graphicx}
\usepackage{booktabs}
\usepackage{algorithm}
\usepackage{algpseudocode}
\usepackage[many]{tcolorbox}
\usepackage{xcolor}
\usepackage{hyperref}
\usepackage{verbatim}

\usepackage{amssymb} 

\hypersetup{colorlinks=true, linkcolor=blue, citecolor=blue, urlcolor=blue}
\newsiamremark{remark}{Remark}

\newcommand{\E}{\mathbb{E}}
\newcommand{\N}{\mathcal{N}}
\newcommand{\R}{\mathbb{R}}
\newcommand{\TSI}{\textsc{TSI}}
\newcommand{\sTSI}{\hat{s}_{\TSI}}

\newcommand{\sTWD}{\hat{s}_{\textsc{TWD}}}
\newcommand{\sBLND}{\hat{s}_{\textsc{BLEND}}}
\newcommand{\Var}{\mathrm{Var}}
\newcommand{\Cov}{\mathrm{Cov}}
\newcommand{\dd}{\mathrm{d}}
\newcommand{\tr}{\mathrm{Tr}}

\renewcommand{\epsilon}{\varepsilon}
\newcommand{\LRp}[1]{\left( #1 \right)}
\newcommand{\LRs}[1]{\left[ #1 \right]}
\newcommand{\LRc}[1]{\left\{ #1 \right\}}
\newcommand{\nor}[1]{\left\| #1 \right\|}

\newcommand{\by}{y}

\DeclareFontShape{OT1}{cmr}{m}{scit}{<->ssub*cmr/m/sc}{}

\newtcolorbox{identitybox}{
 colback=gray!5!white,
 colframe=gray!75!black,
 fonttitle=\bfseries,
 arc=2mm,
 boxrule=1pt
}

\title{Variance Reduced Diffusion Sampling via Target Score Identity}

\author{%
 Alois Duston%
 \thanks{The Oden Institute for Computational Engineering and Sciences,
 The University of Texas at Austin (\email{alois.duston@utexas.edu}).}%
 \and
 Tan Bui Thanh%
 \thanks{The Oden Institute for Computational Engineering and Sciences, and
 the Department of Aerospace Engineering \& Engineering Mechanics, 
 The University of Texas at Austin (\email{tanbui@oden.utexas.edu}).}%
}

\date{July 2025}

\usepackage[numbers,sort&compress]{natbib}
\usepackage{algorithm}
\usepackage{algpseudocode}
\usepackage{tikz}
\usepackage{pgfplots}
\pgfplotsset{compat=1.18}
\usepackage[T1]{fontenc}
\usepackage{lmodern} 

\usepackage{cleveref}
\DeclareMathOperator{\diag}{diag}

\begin{document}
\DeclareFontShape{OT1}{cmr}{m}{scit}{<->ssub*cmr/m/sc}{}
\newcommand{\EE}{\mathbb{E}}
\newcommand{\mc}[1]{\mathcal{#1}}
\newcommand{\mb}[1]{\mathbf{#1}}
\newcommand{\mbb}[1]{\mathbb{#1}}
\newcommand{\mi}[1]{\mathit{#1}}

\newcommand{\bs}[1]{\boldsymbol{#1}}

\maketitle

\begin{abstract}
We study variance reduction for score estimation and diffusion based sampling in settings where the clean (target) score is available or can be approximated.
We start from the Target Score Identity (TSI), which expresses the noisy marginal score as a conditional expectation under the forward diffusion kernel. Building on this, we develop:
(i) a nonparametric estimator based on self normalized importance sampling that can be used directly with standard solvers
(ii) a variance-minimizing \emph{state and time dependent} blending rule between Tweedie type and TSI estimators together with an anticorrelation analysis,
(iii) a data-only extension based on locally fitted proxy scores, and
(iv) a 
likelihood informed extension to Bayesian inverse problems.
Experiments on synthetic targets and PDE governed inverse problems demonstrate improved sample quality for a fixed simulation budget.
\end{abstract}

 
\vspace{-1.0em}
\section{Introduction}
\begin{table}[H]
\centering
\label{tab:notation}
\begin{tabular}{ll}
\toprule
Symbol & Meaning \\
\midrule
$d$ & ambient dimension of data\\
$x_0 \in \R^d$ & clean state with density $p_0$ \\
$x_t \in \R^d$ & forward-diffused state at time $t$ with density $p_t$ \\
$y \in \R^d$ & generic evaluation point for the time-$t$ state (often $y=x_t$) \\
$p_{t|0}(x_t\mid x_0)$ & OU forward transition density (Mehler kernel) \\
$s_t(z) := \nabla_z \log p_t(z)$ & time-$t$ score; in particular $s_0(x_0)=\nabla_{x_0}\log p_0(x_0)$ \\
$N_{\text{ref}}$ & reference-set, size $N_{\text{ref}}$\\
$\{x_0^i\}_{i=1}^{N_{\text{ref}}}$ & reference set (typically i.i.d.\ samples from $p_0$) \\
$\sTWD(y,t)$ & Tweedie score estimator, formed from $\{x_0^i\}_{i=1}^{N_{\text{ref}}}$
\\
$\sTSI(y,t)$ & \TSI $ $ score estimator, formed from  $\{x_0^i, s_0(x_0^i)\}_{i=1}^{N_{\text{ref}}}$  \\
$\lambda(y,t)\in[0,1]$ & state- and time-dependent blending weight \\
$\sBLND(y,t)$ & blended score estimator, given $\lambda\sTWD + (1 - \lambda)\sTSI$\\
\bottomrule
\end{tabular}
\end{table}

Diffusion and flow models have achieved strong empirical performance across modalities by learning the \emph{score function} $s_t = \nabla\log p_t$ of a probability density $p_t$ along a decreasing noise trajectory and integrating a reverse dynamics to synthesize samples \citep{ho2020denoising,song2021score,dhariwal2021diffusion,lipman2022flow,liu2023flow,song2023consistency}.
Despite rapid progress in output realism and generation speed from score based generative models,  a central challenge is \emph{sampling fidelity}: the ability of a sampler to faithfully resolve fine scale geometric structure (e.g., thin manifolds) and to correctly represent separated modes with the right relative weights \citep{ho2020denoising,song2021score,dhariwal2021diffusion,karras2022elucidating}. In practice, these fine scale density features are concentrated at small diffusion times, precisely where the standard mechanism for score estimation (e.g., Tweedie type denoising estimators) becomes ill conditioned and high variance. This can manifest as oversmoothing/mode imbalance for learned denoisers, or as sample memorization and local density fragmentation for nonparametric SNIS baselines \citep{song2021score,karras2022elucidating}.

Most existing strategies for managing score estimator variance fall into two categories: architectural methods that embed inductive biases into network designs \citep{ho2020denoising,nichol2021improved,dhariwal2021diffusion}, and sampler specific accelerations such as DDIM \citep{song2021ddim} and high order ODE solvers \citep{lu2022dpmsolver,karras2022elucidating}. 
Our perspective is orthogonal to these directions: we improve the statistical estimator of the score field itself, pointwise at the state $y$ and time $t$, so that any downstream sampler or distillation scheme inherits lower variance and achieves higher sampling fidelity.

A key challenge is that an \emph{unbiased identity for the score does not by itself yield a low-variance sampling algorithm}.
In practice, common unbiased estimators fail in complementary regimes: the Tweedie estimator can be noisy at small diffusion times, while TSI-based estimators \cite{debortoli2024tsm} can be noisy at large diffusion times or when the target density is nearly singular.
This motivates an \emph{adaptive blending} viewpoint: we treat each identity as producing an inference-time score estimate and combine them with a state-dependent weight $\lambda(y,t)$ chosen to reduce variance through error cancellation.
In contrast to \citet{debortoli2024tsm}, who observes that convex combinations of unbiased identities remain valid and uses such combinations to construct training objectives, we develop and justify a principled inference-time framework that \emph{selects} $\lambda(y,t)$ to approximately minimize estimator variance by exploiting the estimators' error structure.
The resulting blended score is designed to drop into standard reverse-time ODE/SDE solvers (and related distillation procedures) as a plug-in replacement, without changing the model architecture or the integrator. All of our nonparametric inference-time score estimators are implemented via \emph{self-normalized importance sampling (SNIS)} over a fixed reference set.

Concretely our main contributions in this work are as follows:
\begin{itemize}
 \item First, in \cref{sect:CSE,subsec:optimal_blending}, we introduce a blending framework that operates directly on score estimators for diffusion sampling: a nonparametric TSI score estimator computed from a fixed reference set, together with a state- and time-dependent convex combination with the Tweedie score estimator chosen by conditional variance minimization from plug-in SNIS (co)variance estimates at $(y,t)$.

 \item Second, in \cref{sect:negativeCorrelation}, we analyze the error structure underlying this blend and prove exact anticorrelation of the Monte Carlo errors in the linear--Gaussian case; more generally, we establish negative correlation for small diffusion time under regularity conditions on the target distribution, explaining when variance-minimizing blending yields the largest error reduction.

 \item Third, in \cref{sec:learned_proxy} and \cref{app:score_proxy_details}, we generalize the method to data-only settings by introducing a learned local Gaussian proxy for the unavailable initial score and showing that the same variance-minimizing blending machinery applies with this proxy.

 \item Finally, in \cref{sec:bayes-inverse}, we extend the framework to Bayesian inverse problems by incorporating the likelihood into the SNIS weights, yielding posterior versions of the score estimators that integrate into the same reverse-time solvers without modifying the sampler.


\end{itemize}

\subsection*{Note on Concurrent work}
After the publication of the first draft \cite{DustonBuiThanh2026} of the current manuscript, an author of \cite{kahouli2025cvsm} brought to our attention their concurrent work, which also addresses variance reduction via blending unbiased score identities.
While both works exploit complementary variance profiles, they diverge in methodology and scope:
the work in \cite{kahouli2025cvsm} derives an optimal \emph{time dependent} control coefficient (in expectation), whereas we derive a \emph{state and time dependent} blending weight estimated from plug-in SNIS (co)variances at $(y,t)$.
This state-time dependence allows the estimator to adapt locally to geometry of the distribution rather than applying a single global correction per noise level.
Preliminary versions of our blending rule and variance analysis were presented publicly prior to the appearance of \cite{kahouli2025cvsm}: see the dated forum record \cite{duston2025studentforum_event} and accompanying slides \cite{duston2025studentforum_slides}.

\section{Relation to Prior Work}

Diffusion and score based generative models have become a dominant approach for sampling by learning noise conditional scores and reversing a corruption process. Denoising diffusion probabilistic models (DDPM) introduced the modern denoising formulation \cite{ho2020ddpm}, while the SDE view unified score based diffusion with reverse time dynamics and predictor corrector 
samplers \cite{song2021score}. Subsequent work improved performance and scalability through architecture and training refinements \cite{nichol2021improved,dhariwal2021diffusion}.

Building on this foundation, a major line of work focuses on reducing the number of function evaluations by accelerating integration or imposing consistency across noise levels. Training free or post hoc acceleration includes Denoising Diffusion \emph{Implicit} Models (DDIM) \cite{song2021ddim} and high order ODE solvers such as the Diffusion Probabilistic Model Solver (DPM Solver) \cite{lu2022dpmsolver}. Alternative training paradigms learn vector fields directly via Flow Matching \cite{lipman2022flow} or Rectified Flow \cite{liu2022rectifiedflow}, and Consistency Models impose algebraic relations across noise levels to enable one- or few step generation \cite{song2023consistency}.
\emph{Our approach is complementary and orthogonal:} rather than proposing a new solver or a heuristic consistency constraint, we improve the \emph{statistical estimator of the score field itself} at a fixed $(y,t)$. This estimator drops into any standard reverse Stochastic Differential Equation (SDE), Ordinary Differential Equation (ODE) integrator or consistency/distillation pipeline and, by provably lowering pointwise variance through optimal blending, can improve sample quality for a fixed simulation budget.

Our framework is based on the TSI identity. First established  in the Target Score Matching (TSM) literature, \cite{debortoli2024tsm}, TSI  relates the noisy marginal score $s_t$ to a conditional expectation of the clean/target score $s_0$ under the forward diffusion. This identity motivates TSM losses for score matching when the initial score $s_0$ is available.
In contrast, we study TSI through the lens of inference-time variance: we characterize the variability of nonparametric TSI estimators when they are evaluated repeatedly inside reverse-time integrators.
Several recent works use TSI style objectives to train diffusion samplers from unnormalized densities, including Particle Denoising Diffusion Samplers (PDDS) \cite{phillips2024pdds} and Iterated Denoising Energy Matching (iDEM) \cite{akhound2024idem}, with scalable variants such as Adjoint Sampling \cite{havens2025adjoint}.
Very recent work has cast Tweedie's Identity and TSI as a \emph{control variate} family (CVSI) and derived an optimal \emph{time dependent} control coefficient that minimizes variance in expectation \cite{kahouli2025cvsm}.
Our work is complementary: we focus on \emph{state and time dependent} variance minimization, provide an explicit anticorrelation mechanism and diagnostics, and develop extensions to data-only score proxies, Bayesian inverse problems, and efficient neural distillation of the blended score estimator.

Another active direction of research seeks improvements not from SDE solvers but from the structure of the governing equations, using the score's Fokker-Planck equation as a source of regularization or supervision \cite{lai2023fpdiffusion,hu2024scorepinn,zhou2025mfcontrol}.
Several works use the score Fokker-Planck (FP) equation to regularize denoising score matching
 (e.g., FP Diffusion) \cite{lai2023fpdiffusion}, and Score PINNs minimize the residual of the score PDE directly \cite{hu2024scorepinn}. Mean field/control formulations similarly connect sampling to forward PDEs \cite{zhou2025mfcontrol}. In contrast, we derive exact finite-time identities for affine diffusion processes by exploiting the closed-form transition kernels of these SDEs. These identities yield well-conditioned supervision at small times. We focus on estimator variance, proving that the Monte Carlo errors of the TSI and Tweedie estimators are negatively correlated under regularity conditions on $p_0$. Empirically, when the same reverse integrator is used, our variance-minimizing blend attains higher sample quality because the local score estimates have lower error at the points where they are evaluated.

Separate from score PDE theory, nonparametric score estimation has deep connections to Reproducing Kernel Hilbert Space (RKHS) methods, including kernel exponential families and kernelized score matching \cite{hyvarinen2005score,sriperumbudur2017infinite,arbel2018kcef,%
wenliang2019deepkef,zhou2020nonparametricscore}. 
Kernelized samplers such as Stein Variational Gradient Descent (SVGD) and Kernel Stein Discrepancy (KSD) or Maximum Mean Discrepancy (MMD) flows transport particles using functionals of the \emph{target} score or discrepancy \cite{liu2016svgd,liu2016ksd,%
chwialkowski2016kernelgof,korba2021ksddescent,arbel2019mmdflow}. 
In this work, we do not use kernels to fit a parametric density, score or transport map, nor do we assume access to the exact \emph{time marginal} score $\nabla \log p_t$. Instead, we construct \emph{kernel weighted, nonparametric} estimators of the time marginal score that are PDE exact for any affine forward diffusion processes (via Tweedie and TSI) and then combine them by variance optimal blending. The benefit is statistical, lower estimator error at the query $(y,t)$, and thus generically compatible with all standard samplers.

Finally, these score based tools have increasingly been deployed beyond unconditional generation, serving as priors for posterior inference in imaging and scientific inverse problems \cite{feng2023scoreprior,adam2022lensing,legin2023cosmic}. Our framework contributes an estimator that reduces variance by exploiting the closed-form transition kernels of affine diffusions. By reweighting the SNIS weights with the likelihood, this estimator converts prior score estimates into posterior ones without altering the reverse integrator. This preserves the efficiency gains of the blended estimator $\sBLND$ \cref{eq:snis_blended_score} while changing only the weighting.

\section{Theory: From Exact Identities to Optimal Estimators}

\subsection{Score Based Sampling with the Ornstein Uhlenbeck Process}
\label{subsec:ou_sampling}
In the following, we use the typical notation in that random variables are denoted by capital letters, while lowercase letters are for their values.

Score based generative models first define a ``forward process'' that corrupts data with noise over a pseudo time variable $t$ \cite{song2021score,ho2020ddpm}.
 We focus on the Ornstein--Uhlenbeck (OU) process as a canonical worked example for convenience, and because it admits easy closed form transitions \cite{Oksendal98,Gardiner04}. However, the key identities and estimator formulas used below extend to general (time inhomogeneous) affine diffusion processes, with the corresponding derivations collected in \cref{app:affine-CSE}. 
 
 The Ornstein--Uhlenbeck (OU) process is defined by the following Stochastic Differential Equation (SDE):
\begin{equation}
\dd X_t = -X_t \dd t + \sqrt{2} \dd W_t, \quad X_0 \sim p_0,
\label{eq:OUprocess}
\end{equation}
where $X_0 := X_{t=0}$ is distributed according to the data distribution $p_0$.
The OU SDE in \cref{eq:OUprocess} has the closed form forward update \cite{Oksendal98,Gardiner04}
\begin{equation*}
x_t = e^{-t}x_0 + \sqrt{1 - e^{-2t}} \epsilon, \quad \epsilon \sim \N(0, I).
\end{equation*}
We denote the (Gaussian) transition kernel by
\begin{equation}
p_{t|0}(x_t\mid x_0) = \N\!\big(x_t;\, e^{-t}x_0,\,(1 - e^{-2t})I\big).
\label{eq:OU_kernel}
\end{equation}
The time-$t$ marginal is then given by the convolution
\begin{equation}
p_t(x_t) = \int p_{t|0}(x_t\mid x_0)\,p_0(x_0)\,\dd x_0.
\label{eq:OU_marginal}
\end{equation}
We define the time-$t$ \emph{score function} by
\begin{equation}
s(x,t) \coloneqq \nabla_x \log p_t(x).
\label{eq:OU_score_def}
\end{equation}
The corresponding \emph{OU posterior} of the earlier state $x$ given the latter state $y$ is
\begin{equation}
p_{t|0}(x\mid y)\;=\;\frac{p_0(x)\,p_{t|0}(y\mid x)}{p_t(y)}.
\label{eq:OU_posterior}
\end{equation}
In particular, for any test function $f$ we have
$\E[f(X_0)\mid X_t{=}y]=\int f(x)\,p_{t|0}(x\mid y)\,dx$.
We will use the shorthand ``$p_{t|0}(x_0\mid y)$'' throughout to denote the OU posterior.

As $t$ increases, the (marginal) distribution of $X_t$, denoted by $p_t(x)$, smoothly approaches a standard normal distribution \cite{song2021score}.
The generative task is to reverse this process. This is possible by solving the corresponding time reversal SDE \cite{song2021score}:
\begin{equation*}
\dd X_t = [X_t + 2s(X_t,t)]\dd t + \sqrt{2}\dd\bar{W}_t,
\end{equation*}
where $\dd t$ is a positive time step for the backward process. If we can accurately estimate the score function $s_t := s(\cdot,t)$, we can reverse the diffusion to generate new data. This is the premise of all Denoising Score Matching (DSM) models \cite{hyvarinen2005score,vincent2011connection,song2021score,ho2020ddpm}.

\subsection{The Tweedie Identity and Denoising Score Matching}
\label{sect:tweedie}
Denoising score matching \cite{vincent2011connection,hyvarinen2005score} connects the problem of learning a score function to that of denoising corrupted data. The main tool used for this approach is Tweedie's formula \cite{efron2011tweedie,robbins1956empirical}, which provides an exact expression for the score of a noise corrupted density in terms of a conditional expectation. For OU, this identity takes the form:
\begin{equation*}
 s(y, t) = -\frac{1}{1 - e^{-2t}}\mathbb{E}_{x_0 \sim p_{t|0}(\cdot|y)}\big[y - e^{-t}x_0\big],
\end{equation*}
where the conditional expectation is taken with respect to the OU posterior $p_{t|0}(x_0\mid y)$ defined in \cref{eq:OU_posterior}.
To be concrete,  $p_0$ denotes the data distribution at time $0$, the OU forward transition admits the Gaussian transition kernel 
\begin{equation}
p_{t|0}(y\mid x_0)
= \mathcal{N}\!\big(y;\, e^{-t}x_0,\,(1-e^{-2t})I\big)
\propto \exp\!\left(-\frac{\|y-e^{-t}x_0\|^2}{2(1-e^{-2t})}\right).
\label{eq:OU_transition_kernel}
\end{equation}

Given a reference set of particles $\{x_0^i\}_{i=1}^{N_{\text{ref}}} \sim p_{0}$, we can form a nonparametric Tweedie estimator for the score using self normalized importance sampling (SNIS) \cite{owen2013mc,robert2004montecarlo}, as follows
\begin{equation}
\sTWD(y, t) = -\frac{1}{1 - e^{-2t}} \sum_{i=1}^{N_{\text{ref}}} \tilde{w}_{i}(y, t)\big(y - e^{-t}x_0^i\big),
\label{eq:nonParametricTweedie}
\end{equation}
where the (unnormalized) importance weights are defined by evaluating the OU transition kernel at the reference particles,
\begin{equation}
w_i(y,t) := p_{t|0}(y\mid x_0^i),
\qquad
\tilde{w}_i(y,t) := \frac{w_i(y,t)}{\sum_{j=1}^{N_{\text{ref}}} w_j(y,t)}.
\label{eq:OU_transition_weights}
\end{equation}
For the remainder of the paper, any importance weights denoted $w_i$ (and their normalized versions $\tilde w_i$) refer to the OU transition weights \cref{eq:OU_transition_weights} unless explicitly stated otherwise.

\subsection{The Target Score Identity (TSI)}
\label{sect:CSE}
In addition to Tweedie's identity, the denoising score matching literature provides a second unbiased route to the time-$t$ score based on conditional expectations under the forward diffusion. In the Target Score Matching (TSM) framework \cite{debortoli2024tsm}, this appears as the \emph{Target Score Identity (TSI)}, which expresses the score of the time-$t$ marginal as a scaled conditional expectation of the clean score under the forward diffusion posterior. 

We use TSI as a complementary score estimator: it avoids the $\sigma_t^{-2}$ amplification that makes Tweedie unstable as $t \to 0$, so it is better conditioned at low noise.

\begin{lemma}[Target Score Identity (TSI \cite{debortoli2024tsm})]
\label{lem:CSE}
Let $p_0$ be a distribution on $\R^d$, let $p_{t|0}(\cdot\mid y)$ denote the OU posterior defined in \cref{eq:OU_posterior}, and let $s(\cdot,t)$ denote the score of the time-$t$ marginal $p_t$ defined in \cref{eq:OU_score_def}. Then, for any $t>0$,
\begin{equation}
\boxed{
s(y,t) \;=\; e^{t}\,\E_{x\sim p_{t|0}(\cdot\mid y)}\!\big[s_0(x)\big].
}
\label{eq:CSEidentity}
\end{equation}
\end{lemma}

\begin{proof}
A self-contained derivation of the generalized Target Score Identity for general (time-inhomogeneous) linear/affine SDEs of the form
\[
dX_t = A(t)X_t\,dt + b(t)\,dt + G(t)\,dW_t
\]
is given in \cref{app:affine-CSE}. 
This Appendix also presents the gradient--semigroup commutation perspective on the TSI identity and specializes TSI to the canonical OU, Variance Preserving (VP), and Variance Exploding (VE) forward diffusion processes  as special cases.
\end{proof}

The Target Score Identity \cref{eq:CSEidentity} is constructive: given reference particles and scores $\{x_0^i, s_0(x_0^i)\}_{i=1}^{N_{\text{ref}}}$, we approximate the conditional expectation with a self normalized importance sampling (SNIS) average, which yields the nonparametric TSI estimator
\begin{equation} 
\sTSI(y, t) = e^{t} \sum_{i=1}^{N_{\text{ref}}} \tilde{w}_{i}(y, t)\,s_{0}(x_0^i).
\label{eq:nonParametricCSE}
\end{equation} 
This estimator applies when the initial score $s_0(x)=\nabla_x\log p_0(x)$ is either known analytically or can be accurately approximated, and is sufficiently regular that the estimator variance is controlled. This regime covers many problems in scientific computing. For instance, in molecular dynamics $s_0$ can be computed from a known potential function \citep{allen2017computer,frenkel2001understanding}, and in PDE constrained inverse problems it can be computed using adjoint methods \cite{Bui-ThanhGirolami14, LanEtAl2016, Bui-ThanhNguyen2016, Bui-ThanhGhattas12a, Bui-ThanhGhattas12,Bui-ThanhGhattas12f, Bui-ThanhGhattasMartinEtAl13, Bui-ThanhBursteddeGhattasEtAl12_gbfinalist, Bui-ThanhGhattas15}.

In data-only diffusion models, where $p_0$ is available only through samples, one cannot directly exploit access to $s_0$ and the standard route is through denoising identities and their Tweedie type reparameterizations \cite{song2021score}. When $s_0$ \emph{is} available, the TSI/TSM literature explicitly leverages it, particularly to improve conditioning at low noise \cite{debortoli2024tsm}. Importance weighted TSI estimators like \cref{eq:nonParametricCSE} can exhibit large variance at high noise levels, which makes it difficult to obtain competitive samplers without combining them with other estimators at large noise regimes. Indeed, prior work discusses convex combinations of denoising and target score identities as a practical mechanism for constructing alternative score identities and objectives \cite{debortoli2024tsm}. In this work we treat this combination as a conditional error minimization problem and derive a state and time dependent blending rule from plug-in SNIS (co)variances, which yields a principled estimator level blending principle. This viewpoint is complementary to the contemporaneous control variate formulation in \cite{kahouli2025cvsm}.


\subsection{Optimal Blending of Complementary Estimators}
\label{subsec:optimal_blending}

The Tweedie estimator \cref{eq:nonParametricTweedie} and the TSI estimator \cref{eq:nonParametricCSE} converge to the same true score but have two important complementary finite sample properties. In particular, \cref{sect:oppositeScaling} discusses their opposite variance growth and decay with pseudo time $t$, and \cref{sect:negativeCorrelation} shows that their finite sample errors are negatively correlated. In \cref{sect:optimalBlend}, we exploit the negative correlation in their sample errors to provide a variance minimal optimal convex blending of the two estimators.

\subsubsection{Opposite growth and decay of the two estimators}
\label{sect:oppositeScaling}
The Monte Carlo variances of the TSI \cref{eq:nonParametricCSE} and Tweedie \cref{eq:nonParametricTweedie} score estimators scale in opposite directions with diffusion time $t$: TSI is best conditioned at small $t$, while Tweedie is best conditioned at large $t$.
From \cref{eq:nonParametricCSE,eq:nonParametricTweedie} we obtain the variance scalings
\begin{equation}
\Var[\hat{s}_{\TSI}] \;\propto\; \frac{e^{2t}}{N_{\text{ref}}},
\qquad
\Var[\hat{s}_{\text{TWD}}] \;\propto\; \frac{e^{-2t}}{N_{\text{ref}}(1-e^{-2t})^2}.
\label{eq:variance_scalings}
\end{equation}
These rates directly imply complementary time regime behavior. As $t\to 0$, we have
$1-e^{-2t}\sim 2t$, so
\[
\Var[\hat{s}_{\text{TWD}}] \;\propto\; \frac{e^{-2t}}{N_{\text{ref}}(1-e^{-2t})^2}
\;\sim\; \frac{1}{4N_{\text{ref}}\,t^2},
\]
which diverges, whereas $\Var[\hat{s}_{\TSI}]\propto e^{2t}/N_{\text{ref}}\to 1/N_{\text{ref}}$ remains bounded.
Conversely, as $t$ increases, $\Var[\hat{s}_{\TSI}]$ grows like $e^{2t}$, while $\Var[\hat{s}_{\text{TWD}}]$ decays like $e^{-2t}$ (since $1-e^{-2t}\to 1$), yielding a stable large-$t$ Tweedie estimate.

\subsubsection{Negative correlation of the two estimators}
\label{sect:negativeCorrelation}
Beyond their opposite variance scaling in $t$ (\cref{sect:oppositeScaling}), the TSI and Tweedie estimators also exhibit \emph{negatively aligned} Monte Carlo errors. We now formalize this phenomenon. In the \emph{linear--Gaussian} case, the anticorrelation is \emph{exact} and purely algebraic.

\begin{proposition}[Gaussian case: exact anticorrelation]
\label{prop:negcorr-gaussian}
Assume that $p_0=\mathcal N(\mu_0,\Sigma)$ with $\Sigma\succ0$. For a given $(y,t)$, let
\[
  \hat{\mu}_{\text{SNIS}} \;:=\; \sum_{i=1}^{N_{\text{ref}}} \tilde w_i\,x_0^i,\qquad
  \mu \;:=\; \E_{p_0}\LRs{\widehat{\mu}_{\text{SNIS}}}, \qquad 
  \Delta \;:=\; \hat{\mu}_{\text{SNIS}} - \mu.
\]
Let $\widehat s_{\mathrm{TSI}}(y,t)$ and $\widehat s_{\mathrm{TWD}}(y,t)$ be the nonparametric TSI and Tweedie estimators, and let $s(y,t)$ denote the true time-$t$ score. Their errors
\begin{align*}
\varepsilon_{\mathrm{C}} := \widehat s_{\mathrm{TSI}} - \E_{p_0}\LRs{\widehat s_{\mathrm{TSI}}} = -\,e^{t}\,\Sigma^{-1}\,\Delta, 
  \varepsilon_{\mathrm{T}} := \widehat s_{\mathrm{TWD}} - \E_{p_0}\LRs{\widehat s_{\mathrm{TWD}}} = \frac{e^{-t}}{\,1-e^{-2t}\,}\,\Delta.
\end{align*}
and hence the trace of their covariance $\operatorname{Tr}\LRs{\Cov(\varepsilon_{\mathrm C},\varepsilon_{\mathrm T})}$ is given by
\begin{equation}
 \operatorname{Tr}\LRs{\Cov_{p_0}(\varepsilon_{\mathrm C},\varepsilon_{\mathrm T})} = \E_{p_0}\!\big[\varepsilon_{\mathrm T}^{\!\top}\varepsilon_{\mathrm C}\big]
  \;=\;
  -\,\frac{1}{\,1-e^{-2t}\,}\;\E_{p_0}\!\big[\Delta^{\!\top}\Sigma^{-1}\Delta\big]
  \;\le 0.
  \label{eq:inner-product-negative}
\end{equation}
with equality iff $\Delta=0$. 
\end{proposition}
\begin{proof}
See \cref{app:proof:negcorr-gaussian}.
\end{proof}

\begin{remark}
The scalar correlation \cref{eq:inner-product-negative} provides a theoretical justification explaining why the optimal blending reduces the variances, as it appears as a critical component in the variances of the blend (see \cref{prop:OptimalVariance} ). We further note that we have used the means $\E_{p_0}\LRs{\widehat s_{\mathrm{TSI}}}$ and $\E_{p_0}\LRs{\widehat s_{\mathrm{TWD}}}$ to compute the deviations $\varepsilon_{\TSI}$ and $\varepsilon_T$ both TSI and Tweedie and show that the deviations are anti correlated. We can replace SNIS mean with the exact mean $\mu = \E\LRs{X_0|X_t = y}$, and \cref{prop:negcorr-gaussian} still holds. In this case, the result says that errors in TSI and Tweedie estimators are anti correlated.
For sufficient large sample size $N_{\text{ref}}$, the biases (due to SNIS) in both TSI and Tweedie estimators are small 
, using the exact mean (the exact score, respectively) or SNIS mean (SNIS score mean, respectively) are thus asymptotically the same. 
\end{remark}

A general extension of \cref{prop:negcorr-gaussian} to non-Gaussian $p_0$ and arbitrary $t$ is currently not available. That said, for sufficiently small $t$, the negative-correlation property can be established under standard regularity assumptions on $p_0$, as stated in \cref{thm:negcorrGeneral}. 

\begin{theorem}[Negative correlation for small time $t$ and large $N_{\text{ref}}$]
\label{thm:negcorrGeneral}
 Suppose the operator norm of the Hessian and the third order derivative tensor of $\log p_0(x)$, the derivative of the score of $p_0$, is bounded as follows:
 \[
 m\LRp{y} I \preceq -\nabla^2_{x}[\log p_0(x)], \quad \nor{\nabla^3_{x}[\log p_0(x)]}_{op} \le c < \infty, \quad \forall x \in \text{supp}\LRp{p_{t|0}},
    \]
 where $\nor{\cdot}_{op}$ is the corresponding operator norm. Assume that 
 $\Sigma_{\text{eff}}^{-1} = -\nabla^2\log p_0\LRp{\mu} \succeq 0$ and that the importance weights are uniformly bounded as $0 < w_{\text{min}} \le w_i \le w_{\text{max}}$.
 Then there exists $N_{\text{ref}}^*$ and $t^*$ such that 
 \[
    \E_{p_0}\!\big[\varepsilon_{\mathrm T}^{\!\top}\LRp{y,t}\varepsilon_{\mathrm C}\LRp{y,t}\big] < 0, \quad N_{\text{ref}} \ge N_{\text{ref}}^* \text{ and } t \le t^*.
    \]
\end{theorem}
\begin{proof}
 From \cref{eq:CandTcorrelation} together with \cref{lem:D} and \cref{lem:E}, we need to find $\overline{t}$ such that 
 \begin{equation*}
    \label{eq:negCondition}
        \frac{w_{\max}}{N_{\text{ref}}}\; \frac{cd^{3/2}}{\kappa
        ^{3/2}} < \tr\LRp{\Sigma_{\text{eff}}^{-1}\mathrm{Cov}_{p_0}\!\left(\widehat{\mu}\right)}.
    \end{equation*}
From the definition of $\kappa$ in \cref{lem:logconcavity},
\[
\kappa = \frac{1}{1-e^{-2t}} = \frac{1}{2t}\,(1+O(t)) \qquad (t\to 0),
\]
so for sufficiently small $t$ we use the scaling $\kappa \asymp (2t)^{-1}$.
 \[
    \sqrt{\overline{t}^3} = \frac{N_{\text{ref}}}{\sqrt{8}} \frac{\tr\LRp{\Sigma_{\text{eff}}^{-1}\mathrm{Cov}_{p_0}\!\left(\widehat{\mu}\right)}}{w_{\max} c\sqrt{d^3}},
    \]
 and if we use the lower bound in \cref{lem:D}, we need to find $\overline{t}$ such that
 \[
        \frac{w_{\max}}{N_{\text{ref}}}\; \frac{cd^{3/2}}{\kappa
        ^{3/2}} < \frac{w_{\text{min}}}{N_{\text{ref}}} \lambda_{\text{min}}\LRp{\Sigma_{\text{eff}}^{-1}} \nor{\Sigma}_{op},
    \]
 and thus we can eliminate $N_{\text{ref}}$ entirely in the expression of $\overline{t}$ as
 \[
    \sqrt{\overline{t}^3} = \frac{1}{\sqrt{8}}\frac{w_{\min}}{w_{\max}} \frac{\lambda_{\text{min}}\LRp{\Sigma_{\text{eff}}^{-1}} \nor{\Sigma}_{op}}{ c\sqrt{d^3}}.
    \]
 The proof ends by taking
 \(
    t^* = \min\LRc{\frac{1}{2}\log\LRp{1 - \frac{1}{m(y)}}, \overline{t}}.
    \)
\end{proof}

Supplementary empirical validation of this phenomenon in the setting of Gaussian Mixtures (GMMs), including correlation curves across time and the
time dependent variance/bias behavior of the two estimators, are deferred to \cref{app:suppres:cor-var}.

\subsubsection{Optimal blending as variance minimization}
\label{sect:optimalBlend}

Given the complementary growth/decay of the variance profiles and, more importantly, the negative correlation between the TSI \cref{eq:nonParametricCSE} and Tweedie \cref{eq:nonParametricTweedie} estimators, we consider a linear blend. Specifically, for scalar function $\lambda = \lambda(y,t)$ we define
\begin{equation}
\sBLND(\lambda)=\lambda\,\sTWD+(1-\lambda)\,\sTSI,
\label{eq:nonParametricBlend}
\end{equation}
and note that $\sBLND(\lambda)$ is unbiased for the true score $s$, since both $\sTWD$ and $\sTSI$ are unbiased estimators for the true score $s$. The question is how to choose $\lambda$ so that the blend $\sBLND(\lambda)$ retains the complementary strengths of both estimators. Since the variance and correlation depend on $(y,t)$, we choose $\lambda(y,t)$ to minimize the conditional variance of the blended error at $(y,t)$:
\begin{equation}
\lambda^*(y,t)\ \in\ \arg\min_{\lambda\in\R}\ J(\lambda;y,t),
\qquad
J(\lambda;y,t):=\E\!\left[\big\|\lambda\,\varepsilon_T+(1-\lambda)\,\varepsilon_{\TSI}\big\|^2\right],
\label{eq:optimalBlend}
\end{equation}
where $\varepsilon_T:=\sTWD-\E\LRs{\sTWD}$ and $\varepsilon_{\TSI}:=\sTSI-\E\LRs{\sTSI}$.

\begin{proposition}[Variance optimal blending weight]
\label{prop:OptimalVariance}
Define
\[
\sigma_T^2:=\E\|\varepsilon_T\|^2,\qquad 
\sigma_C^2:=\E\|\varepsilon_{\TSI}\|^2,\qquad 
\rho:=\E\langle \varepsilon_T,\varepsilon_{\TSI}\rangle.
\]
If $\sigma_T^2+\sigma_C^2-2\rho\neq 0$, then the minimizer of \cref{eq:optimalBlend} is unique and is given by
\begin{equation}
\lambda^{*}(y,t)
=\frac{\sigma_C^2-\rho}{\sigma_T^2+\sigma_C^2-2\rho},
\qquad
J(\lambda^{*};y,t)
=\frac{\sigma_T^2\,\sigma_C^2-\rho^2}{\sigma_T^2+\sigma_C^2-2\rho}.
\label{eq:lambda-optimal}
\end{equation}
Moreover, when $\rho<0$ (negative alignment of errors), the optimal weight satisfies $0<\lambda^*(y,t)<1$ and the variance reduction is amplified as $\rho$ becomes more negative.
\end{proposition}
\begin{proof}
The proof is straightforward, and is deferred to \cref{app:proof:OptimalVariance}.
\end{proof}

In practice, we do not have access to $\sigma_T$, $\sigma_C$, and $\rho$, at inference time, only to their SNIS plug-in approximations.
Specifically, with SNIS weights $\tilde w_i$ we define
\begin{gather*}
a_i := e^{t} s_0(x_0^i), \qquad 
b_i := -\,\frac{1}{1-e^{-2t}}\,(y-e^{-t}x_0^i), \\
\sTSI = \sum_i \tilde w_i a_i, \qquad 
\sTWD = \sum_i \tilde w_i b_i.
\end{gather*}
We also define centered contributions $\delta a_i=a_i-\sTSI$ and $\delta b_i=b_i-\sTWD$.
The standard SNIS plug-in estimates are given
\begin{equation}
\hat \sigma_C^2=\frac{\sum_i \tilde w_i^{\,2}\|\delta a_i\|^2}{1-\sum_i \tilde w_i^{\,2}},\quad
\hat \sigma_T^2=\frac{\sum_i \tilde w_i^{\,2}\|\delta b_i\|^2}{1-\sum_i \tilde w_i^{\,2}},\quad
\hat \rho=\frac{\sum_i \tilde w_i^{\,2}\langle \delta a_i,\delta b_i\rangle}{1-\sum_i \tilde w_i^{\,2}}\frac{1}{\hat{\sigma}_T\hat{\sigma}_C}.
\label{eq:var_weights}
\end{equation}
Plugging $\hat \sigma_C^2,\hat \sigma_T^2$ and $\hat \rho$ into \cref{eq:lambda-optimal}  yields the approximate blend weight $\widehat{\lambda}\LRp{y,t}$.

\noindent We then define the corresponding SNIS plug-in blended score estimator using $\widehat{\lambda}$:
\begin{equation}
\sBLND(y,t)
\;:=\;
\sBLND\!\Big(\widehat{\lambda}\LRp{y,t}\Big)
\;=\;
\Bigl(1-\widehat{\lambda}\LRp{y,t}\Bigr)\,\sTSI
\;+\;
\widehat{\lambda}\LRp{y,t}\,\sTWD.
\label{eq:snis_blended_score}
\end{equation}

This blended score estimator forms the core of our nonparametric  variance minimizing sampling procedure in \cref{alg:CSE_blend_sampling}.

\begin{algorithm}[H]
\caption{Reverse Sampling optimal blend score }
\label{alg:CSE_blend_sampling}
\begin{algorithmic}[1]
\State \textbf{Input:} Initial sampling particles $\{y_{j}(T)\}_{j=1}^{M} \sim \N(0, I_d)$, time grid $T=t_K > \dots > t_0=0$, reference data $\{x_0^i, s_0(x_0^i)\}_{i=1}^{N_{\text{ref}}}$, with $x_0^i \sim p_0$.
\For{$k=K-1, \dots, 0$}
 \State Let current time be $t_{k+1}$ and target time be $t_k$.
 \For{$j=1, \dots, M$}
 \State Compute $\sBLND\LRp{\widehat{\lambda}\LRp{y_j(t_{k+1}), t_{k+1}}}$ in \cref{eq:snis_blended_score} for particle $y_j$. 
 \State Update particle $y_j(t_k)$ using SDE integrator with $\sBLND\LRp{\widehat{\lambda}\LRp{y_j(t_{k+1})}}$.
 \EndFor
\EndFor
\State \textbf{Output:} Final samples $\{y_j(0)\}_{j=1}^{M}$.
\end{algorithmic}
\end{algorithm}

\subsection{A learned proxy for the initial score}
\label{sec:learned_proxy}

When only i.i.d.\ samples \(X=\{x_0^i\}_{i=1}^{N_{\text{ref}}}\sim p_0\) are available, \cref{alg:CSE_blend_sampling} requires an approximation of the unknown initial score \(s_0(x)=\nabla_x\log p_0(x)\).
We construct a \emph{local Gaussian score proxy} \(\hat s_0(x_0^i)\) at each anchor \(x_0^i\) by fitting a kernel weighted Gaussian to its \(k\) nearest neighbors, a standard local nonparametric construction \citep{silverman1986density,mack1979multivariate,wasserman2006all}.
Let \(\mathcal{N}_k(i)\) be the indices of the \(k\) nearest neighbors of \(x_0^i\) in \(X\), set
\(
  h_i^2 := \max_{j\in\mc{N}_k(i)} \|x_0^i - x_0^j\|^2,
\)
define weights
\(
  k_{ij} \propto \exp\!\big(-\|x_0^i-x_0^j\|^2/(2h_i^2)\big),
\)
and normalize \(\tilde k_{ij} := k_{ij}\big/\sum_{\ell\in\mc{N}_k(i)} k_{i\ell}\).
With the weighted mean \(\mu_i := \sum_{j\in\mc{N}_k(i)} \tilde k_{ij}\, x_0^j\) and a local covariance model \(\Sigma_i\) (below), we set
\vspace{-1.0em}
\begin{equation}
  \hat s_0(x_0^i) \;:=\; \Sigma_i^{-1}(\mu_i-x_0^i).
  \label{eq:proxy_gaussian_score}
\end{equation}

We use two covariance families (and label experiments accordingly):

\begin{enumerate}
\item \textbf{Diagonal proxy (\textsc{Diag}).}
\(
  \Sigma_i^{\textsc{Diag}}
  :=
  \mathrm{diag}\!\big(v_{i,1},\dots,v_{i,d}\big) + \tau_i I
\),
giving
\begin{equation}
  \hat s_0^{\textsc{Diag}}(x_0^i)
  := \big(\Sigma_i^{\textsc{Diag}}\big)^{-1}(\mu_i-x_0^i).
  \label{eq:proxy_diag}
\end{equation}
Here \(v_{i,\ell}\) are local per coordinate variances estimated from the \(k\)NN cloud and \(\tau_i>0\) is a ridge/noise floor parameter.

\item \textbf{Low rank plus diagonal tail proxy (\textsc{LR{+}D}).}
\(
  \Sigma_i^{\textsc{LR{+}D}}
  :=
  V_i\Lambda_i V_i^\top + \mathrm{diag}\!\big(\tau_{i,1},\dots,\tau_{i,d}\big)
\),
giving
\begin{equation}
  \hat s_0^{\textsc{LR{+}D}}(x_0^i)
  := \big(\Sigma_i^{\textsc{LR{+}D}}\big)^{-1}(\mu_i-x_0^i).
  \label{eq:proxy_lrd}
\end{equation}
Here \(V_i\in\R^{d\times r}\) and \(\Lambda_i\in\R^{r\times r}\) capture the leading \(r\)-dimensional local principal subspace (e.g., via weighted Principal Component Analysis(PCA)), and the diagonal tail ensures invertibility.
\end{enumerate}

All further implementation details (weighting/bandwidth choices, optional query time recomputation, and complexity considerations) are deferred to Appendix \ref{app:score_proxy_details}.

\subsection{Application to Bayesian Inverse Problems}
\label{sec:bayes-inverse}
We adapt our framework to posterior sampling in inverse problems. Given a prior $p_0(x)$ and likelihood $p(\by\mid x)$ for an observation $\by$, the posterior is given by
\begin{equation*}
p^{\text{post}}(x\mid \by)\ \propto\ p_0(x)\,p(\by\mid x)\;=:\;p_0(x)\,\mathcal{L}(x).
\end{equation*}
As is standard in inverse problems \cite{Stuart10,KaipioSomersalo05}, it is typically straightforward to sample the prior $p_0$ but not the posterior. We therefore reuse the same prior reference set $\{x_0^i\}_{i=1}^{N_{\text{ref}}}$ and incorporate the likelihood by multiplying each prior weight by the likelihood factor and renormalizing.

We assume that the observation $\by$ depends on the unknown $X_0$ but is independent of the forward OU corruption noise. This includes the common linear--Gaussian case $\by=HX_0+\varepsilon$ with $\varepsilon\sim\mathcal N(0,\Sigma_y)$, and more generally any $\mathcal{L}(x_0)$ that does not involve the OU noise used to generate $X_t$. If this assumption were violated, the likelihood would depend on the diffusion path. This would prevent the factorization below and require joint path space inference.

For a query point $(x,t)$, we incorporate the likelihood into the OU transition weights and renormalize, obtaining posterior normalized importance weights
\vspace{-1.0em}
\cite{owen2013mc,robert2004montecarlo}
\begin{equation}
\label{eq:posterior_alpha_main}
\alpha_i(y,t;\by)
\;:=\;
\frac{w_i(y,t)\,\mathcal{L}(x_0^i)}
{\sum_{j=1}^{N_{\text{ref}}} w_j(y,t)\,\mathcal{L}(x_0^j)},
\qquad
w_i(y,t)=p_{t|0}(y\mid x_0^i).
\end{equation}
Equivalently, $\{\alpha_i(y,t;\by)\}$ are the SNIS weights for expectations under the
likelihood weighted OU posterior $p_{t|0}^{\text{post}}(x_0\mid y; \by)\propto p_0(x_0)\,\mathcal{L}(x_0)\,p_{t|0}(y\mid x_0)$,
approximated using the fixed prior reference set.

Using \cref{eq:posterior_alpha_main} together with the OU transition, we obtain
a family of \emph{posterior} score estimators. First, the posterior initial score
decomposes as
\begin{align*}
s_0^{\text{post}}(x)
:= \nabla_x\log p^{\text{post}}(x\mid \by)
= s_0(x) + \nabla_x\log \mathcal{L}(x).
\end{align*}
This leads to two natural estimators at time $(y,t)$. The Tweedie type estimator is
\begin{align*}
\sTWD^{\text{post}}(y,t)
= -\frac{1}{1-e^{-2t}}\sum_{i=1}^{N_{\text{ref}}} \alpha_i(y,t;\by)\,\bigl(y-e^{-t}x_0^i\bigr),
\end{align*}
while the TSI type estimator replaces the explicit OU drift term by a weighted
average of the posterior initial scores:
\vspace{-1.0em}
\begin{align*}
\sTSI^{\text{post}}(y,t)
= e^{t}\sum_{i=1}^{N_{\text{ref}}} \alpha_i(y,t;\by)\,s_0^{\text{post}}(x_0^i).
\end{align*}

As in the prior (likelihood free) case, the two estimators typically exhibit
negatively correlated Monte Carlo fluctuations when computed from the same SNIS batch.
We therefore form a convex combination with a batch estimated weight chosen to
minimize the plug-in variance using the same coefficients $\alpha_i(y,t;\by)$:
\[
\sBLND^{\text{post}}(y,t)
=(1-\lambda_{\text{snis}}^{\text{post}})\,\sTSI^{\text{post}}(y,t)
+\lambda_{\text{snis}}^{\text{post}}\,\sTWD^{\text{post}}(y,t).
\]
Empirically, the anticorrelation mechanism persists and
is often strongest at intermediate diffusion times (see \cref{app:suppres:cor-var}).


\section{Results}
\label{sec:quant-results}
We present numerical experiments that (i) validate the statistical claims underpinning our framework and (ii) demonstrate how variance reduction translates into improved downstream sampling fidelity. The experiments are organized to move from fully controlled settings, where ground truth scores and errors are accessible, to more challenging inverse problems.

We begin with \emph{results on a low dimensional manifold} (\cref{sec:prior-9d}), using closed form Gaussian mixture models where ground truth scores are available. In this setting we compare Tweedie, TSI, and Blend across quantitative divergence metrics and qualitative PCA projections (2D histograms of samples projected onto fixed principal directions).
 We directly measure Monte Carlo errors as a function of the reference set size $N_{\text{ref}}$ (\cref{fig:rmse_vs_n}, \cref{fig:mmd_ksd_vs_n}). We then conduct the regime sweep in \cref{sec:regime_study} to map out the ``advantage regime'', the combinations of dimension and noise level where the blended score estimator $\sBLND$ \cref{eq:snis_blended_score} yields the largest improvements in curvature and mass fidelity.
 Finally, we turn to inverse problems (\cref{subsec:ns_inverse}, \cref{subsec:mnist}), where we evaluate posterior sampling fidelity under (i) scientific forward operators with exact priors and (ii) image inverse problems with learned score proxies. Full details on hyperparameters, architectures, metric definitions, and experimental setups are provided in \cref{sec:repro}.

Unless stated otherwise, in \emph{all} sampling tests we integrate the reverse-time dynamics with the second order Heun predictor corrector (PC) solver \cite{ascher1997implicit}, a standard choice in score based SDE samplers (see, e.g., \cite{song2021score}); see \cref{sec:repro}, Def.~\ref{def:heun_pc} for the precise update rule. The same solver and time grid are used for Tweedie, TSI, and Blend to ensure comparability.

In our experiments, samplers are evaluated along a log-spaced diffusion time grid $t\in[t_{\min},t_{\max}] = [5{\times}10^{-4},\,1.5]$. The lower bound $t_{\min}$ ensures that the OU noise level $\sigma_t = \sqrt{1-e^{-2t}}$ remains small enough to resolve fine structure while avoiding catastrophic importance weight collapse at small $t$. The upper bound $t_{\max}$ is large enough that $p_t$ is close to the standard normal prior. The log spacing in $t$ allocates more grid points to the small-$t$ regime (small $\sigma_t$), where discretization and score-estimation errors tend to have the largest impact on final sample quality \cite{karras2022elucidating}.
 For TSI (and hence Blend) we estimate conditional expectations via SNIS with an effective sample size(ESS) threshold\footnote{\label{fn:ess_threshold}We quantify importance-sampling quality via the effective sample size $\mathrm{ESS}=1/\sum_i \tilde w_i^2$, where $\tilde w_i$ are normalized SNIS weights. We drop time points with $\mathrm{ESS}<\tau_{\mathrm{ESS}}$, and in all experiments we set $\tau_{\mathrm{ESS}}:=0.05\,N_{\mathrm{ref}}$.}.

\subsection{Moderate dimensional manifold: 9D Helix GMM}
\label{sec:prior-9d}
We test the ability of the samplers to capture a complex, low dimensional manifold embedded in a higher dimensional space. The target is a 9D Gaussian Mixture Model (GMM) whose intrinsic structure is a 3D helix, shown in \cref{fig:prior-9d}. Unless stated otherwise, all quantitative loss curves in this section are computed on this \emph{same 9D Helix GMM}. In the qualitative panels (e.g., \cref{fig:prior-9d}), the point clouds represent samples drawn from the corresponding method (or from the ground truth density) and projected onto the indicated principal directions. Note that we never visualize score vectors directly. For visualization, we project samples onto two orthogonal planes \((d_1,d_2)\), \((d_3,d_4)\), where \(d_1,\ldots,d_4\) denote the first four principal directions obtained by PCA fit to the \emph{target} distribution (fixed once for all methods). This axis selection highlights high variance structure and does not \emph{a priori} favor Blend over Tweedie (or vice versa). All concrete values (number of components, helix pitch/radius, covariance anisotropy, bandwidth grids, SNIS batch sizes, and the $t$-grid) are provided in  \cref{sec:repro}.

\subsubsection*{Quantitative comparisons}
To compare these estimators quantitatively, we use three complementary metrics that emphasize global mass placement, score based discrepancy, and pointwise score error on the \emph{same 9D Helix GMM}: (i) MMD (see Appendix 
\cref{sec:repro}, Def.~\ref{def:mmd}) with an RBF kernel, which primarily reflects \emph{global mass placement and coverage}; (ii) KSD (see Appendix
\cref{sec:repro}, Def.~\ref{def:ksd}) with an inverse multiquadric kernel, a \emph{score based discrepancy} that is sensitive to both location and local geometry through the target score; and (iii) time averaged score root mean squared error (RMSE), (see Appendix
\cref{sec:repro}, Def.~\ref{def:score_rmse}) along the sampling $t$-grid, which measures \emph{pointwise score error} along the diffusion path (ground truth $s$ is available for the GMM). We vary the number of reference samples $N_{\text{ref}}$ to produce the curves in \cref{fig:mmd_ksd_vs_n} and \cref{fig:rmse_vs_n}. Full definitions, estimator details, and kernel/bandwidth choices are deferred to the appendix; implementation details are provided in \cref{sec:repro}.

In the following, by \emph{Blend} (or \emph{Blend Score}, solid blue curves), we mean the results obtained with
the blended score estimator $\sBLND$ \cref{eq:snis_blended_score} using the exact initial score $s_0$; this serves as the oracle reference for the practical \emph{Blend (proxy)} (dashed blue curves), which replaces $s_0$ by the diagonal (Diag) learned local score proxy from \cref{sec:learned_proxy}. We also include the pure TSI estimator (green curves) in \cref{fig:mmd_ksd_vs_n} and \cref{fig:rmse_vs_n}. In \cref{fig:mmd_ksd_vs_n} (left), Blend (proxy) (dashed blue curves) preserves Tweedie's global mass placement accuracy, while in \cref{fig:mmd_ksd_vs_n} (right) it achieves lower KSD than Tweedie, thanks to the local gradient information captured by TSI. In \cref{fig:rmse_vs_n} the RMSE Blend (proxy) (dashed blue) significantly smaller than that of Tweedie. Due to its inability to resolve global structure, the pure TSI score estimator generally performs the worst on transport metrics like MMD. However, its strong performance on local metrics (KSD and score RMSE) confirms that it provides a valid local gradient signal. Unlike Tweedie, which tends to memorize reference samples and fragment the local density, TSI resolves the smooth geometric variations of the underlying manifold, even if it fails to coordinate global mass placement.

\vspace{-.5em}
\begin{figure}[H]
 \centering
 \includegraphics[width=\linewidth]{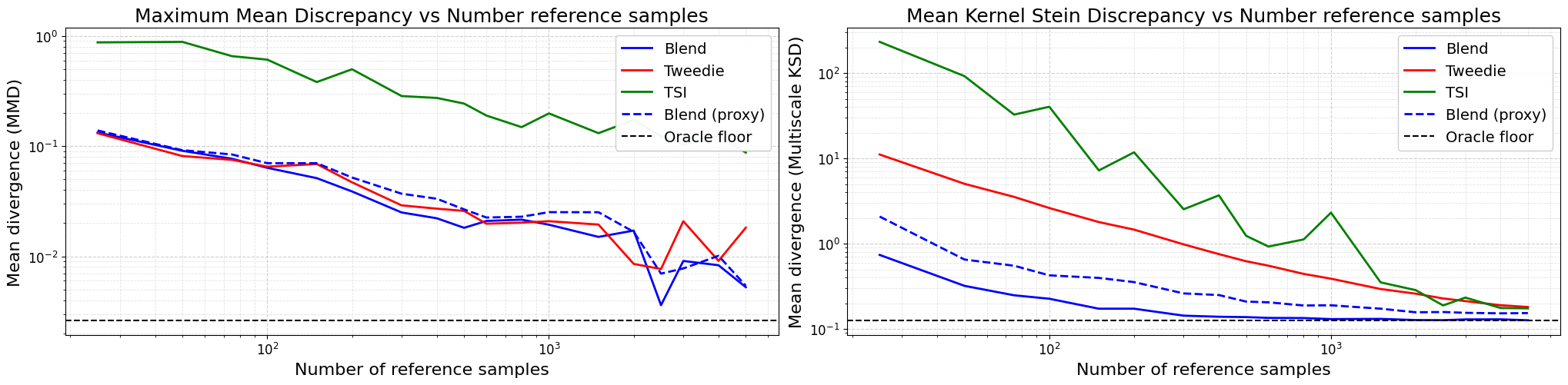}
 \vspace{-1.5em}
 \caption{\textbf{MMD and KSD vs.\ number of references (lower is better)} on the \textbf{9D Helix GMM}.
 \textbf{Left:} MMD with an RBF kernel, reflecting global mass placement and coverage.
 \textbf{Right:} KSD with an inverse multiquadric kernel, a score based discrepancy sensitive to local geometry through the target score.
 Blend (proxy), (dashed blue, using the diagonal learned score proxy from \cref{sec:learned_proxy}) is comparable with Tweedie in terms of global mass placement (left, MMD), while correcting its local score errors (right, KSD). Tweedie's high KSD reflects its tendency to fragment the manifold structure, whereas Blend resolves the local geometry without sacrificing global coverage.
 The Blend (solid blue, using exact $s_0$) further approaches the ground truth floor. While TSI (green) shows high variance, the blended estimators $\sBLND$ \cref{eq:snis_blended_score} stabilize it.}
 \label{fig:mmd_ksd_vs_n}
\end{figure}
\vspace{-1.0em}
\begin{figure}[H]
 \centering
\includegraphics[width=.5\linewidth]{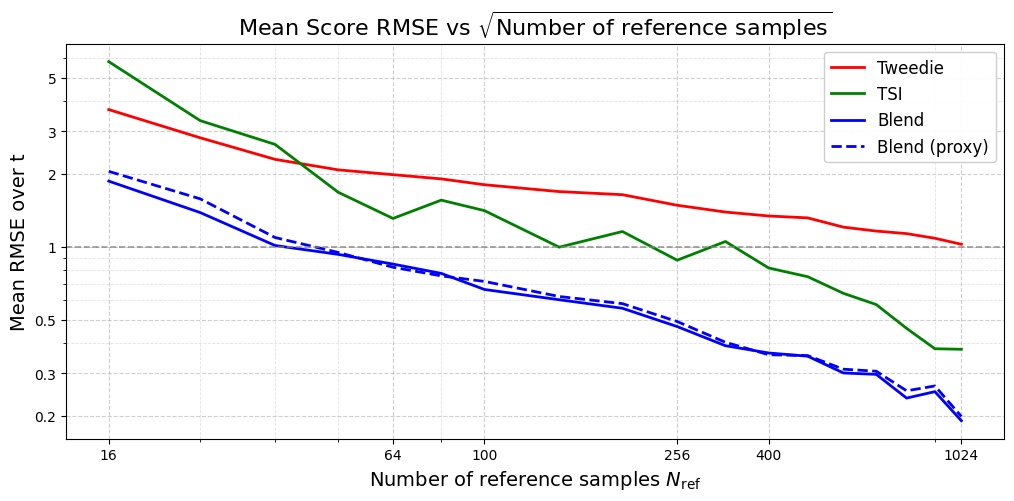}
\caption{\textbf{Time-averaged score RMSE vs.\ $N_{\mathrm{ref}}$} (x-axis uses $\sqrt{N_{\mathrm{ref}}}$ spacing) on the \textbf{9D Helix GMM}.The RMSE of Blend (proxy) (dashed blue; diagonal learned score proxy from \cref{sec:learned_proxy}) closely matches oracle Blend (solid blue; exact $s_0$) and is significantly smaller than the RMSE of Tweedie (red) for all reference sizes.
Tweedie’s elevated RMSE is consistent with overfitting to the reference bank at small diffusion times.
Pure TSI (green) provides an estimator for local manifold curvature (low RMSE) but lacks a global signal to transport mass correctly (see \cref{fig:mmd_ksd_vs_n}).}
\label{fig:rmse_vs_n}
\end{figure}

These three metrics paint a consistent picture
that supports our blending strategy in \cref{sect:optimalBlend}: variance minimal
blending inherits the strengths of both estimators. Most of these gains are preserved even when the TSI term uses a score proxy fitted only to data, indicating that data dependent curvature information extracted from raw samples is sufficient to deliver measurable improvements over Tweedie alone.

\subsubsection*{Qualitative comparison}
We compare five columns: \emph{True} (target samples), \emph{Blend}, which uses the exact target score ($s_0$) inside the variance optimal blend, \emph{Blend (proxy)}, which replaces \(s_0\) by the \textsc{LR{+}D} local Gaussian score proxy from \cref{sec:learned_proxy} fit directly to the raw reference data, \emph{Tweedie} (standard nonparametric baseline), and \emph{TSI} in isolation.
The results in \cref{fig:prior-9d} show that Blend (proxy) closely matches Blend, and both are nearly indistinguishable from the localized, complex ground truth across both PCA marginals.
In contrast, Tweedie typically resolves global scale position information accurately but locally collapses generated samples onto a neighborhood around the reference samples, failing to capture the actual smooth local manifold structure. The TSI estimator, by directly leveraging $s_0$, captures the smooth local variation in the density accurately but its higher variance at larger diffusion times can distort the distribution in noise space, leading to misplaced probability mass after pushing back to $t=0$. The variance optimal Blend resolves these complementary failure modes by combining Tweedie's stable global mass placement with TSI's local smooth geometric fidelity, yielding high quality sampling that neither estimator achieves alone per visual inspection.

\vspace{-1.0em}
\begin{figure}[H]
 \centering
 \includegraphics[width=\linewidth]{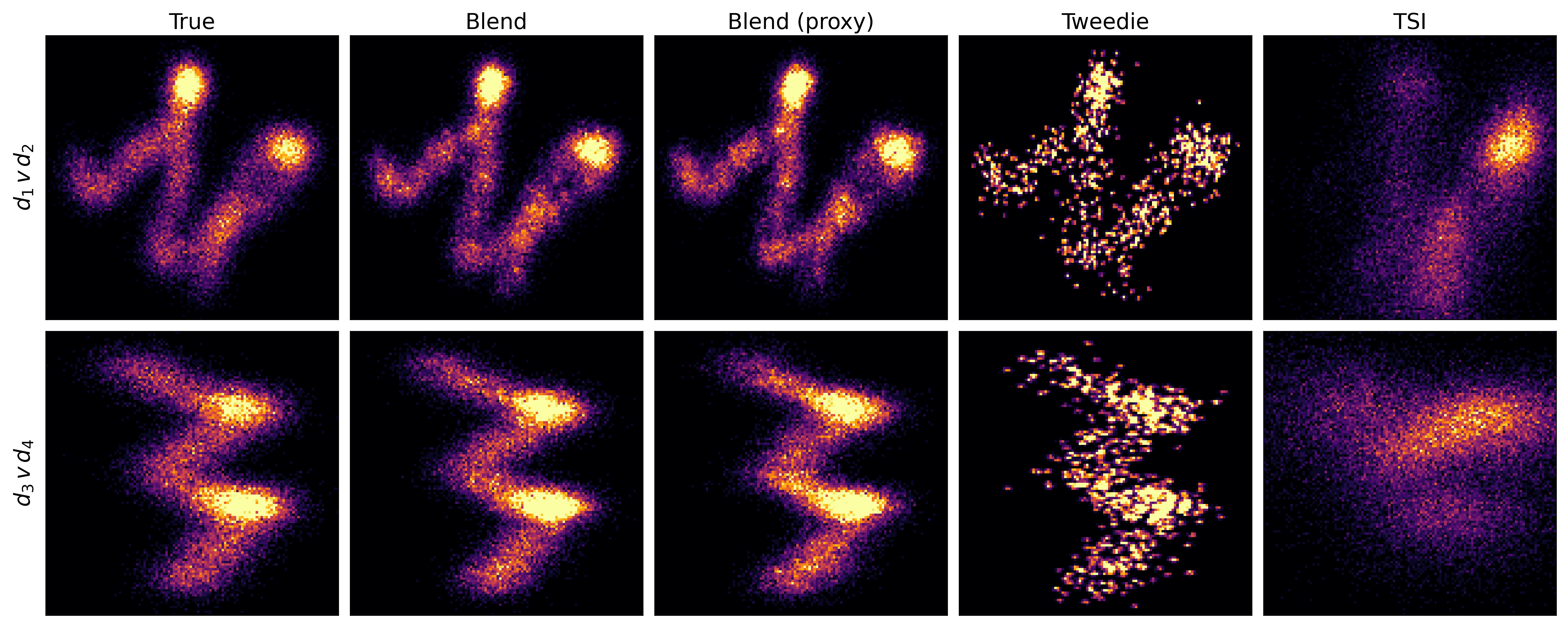}
 \vspace{-.5em}
 \vspace{-1.0em}
 \caption{\textbf{Qualitative comparison on the \emph{9D Helix GMM} (N = 750).}
 Each panel displays a 2D histogram of samples projected onto the principal directions $(d_1,d_2)$ (top row) and $(d_3,d_4)$ (bottom row), with PCA fitted to the \emph{target} distribution and held fixed across methods.
 Columns show the \emph{True} distribution, \emph{Blend}, \emph{Blend (proxy)} using the \textsc{LR{+}D} proxy from \cref{sec:learned_proxy}, \emph{Tweedie}, and \emph{TSI}. Tweedie (4th column) correctly identifies the global region but fragments the manifold, collapsing samples onto the reference particles (memorization). TSI (5th column) captures local curvature but scatters mass due to high variance. The Blend estimators (2nd \& 3rd columns) combine stable global positioning with smooth local reconstruction.}
 \label{fig:prior-9d}
\end{figure}


\subsection{Characterizing the Posterior Sampling Advantage Regime}
\label{sec:regime_study}

Before turning to the PDE and imaging inverse problems (\cref{subsec:ns_inverse,subsec:mnist}), we study the signal to noise ratio (SNR) regimes in which the variance-minimizing blend improves over the Tweedie baseline in a fully controlled setting with an \emph{exact posterior} and \emph{exact score}.
In the regime sweep in ~\cref{fig:regime_sweep},  we find that:
(i) both methods struggle at high SNR (sharply concentrated likelihood) due to weight degeneracy;
(ii) both methods become comparable when data are noisy (posterior close to the prior); and
(iii) there is an intermediate noise, moderate dimension regime where the blend achieves appreciably lower sampling error than the tweedie based sampler.

Importantly, we also expose a \emph{high dimensional failure mode}. In this regime, both approaches weaken, but the blended score estimator $\sBLND$ \cref{eq:snis_blended_score} degrades further than the baseline. This is because the blend is ``double leveraged'' on quantities that become unstable in high dimensions, specifically, the collapse of SNIS weights and the difficulty of resolving optimal variance minimization weights \cref{eq:var_weights}. This finding motivates truncating to the problem to a moderate number dimension in our subsequent experiments to remain in the regime where blending is most advantageous.

\subsection*{Synthetic inverse problem family (linear--Gaussian likelihood with a GMM prior)}
For each dimension $d \in \{3,6,12,24\}$ we consider
\begin{equation}
x \sim p_0(x), 
\qquad 
y_{\mathrm{obs}} \mid x \sim \mathcal N(Ax,\sigma^2 I),
\label{eq:toy_inverse}
\end{equation}
where $p_0$ is a Gaussian mixture prior (GMM) in $\mathbb R^d$, and $A:\mathbb R^d \to \mathbb R^d$ is a fixed linear ``forward operator'' with a non trivial spectrum (chosen to mimic the anisotropy/ill conditioning typical of inverse problems; details are fixed in the sweep script and held constant across the sweep). For each trial we draw a fresh latent truth $x^\star \sim p_0$ and observation noise $\varepsilon\sim\mathcal N(0,I)$ and set
\(
y_{\mathrm{obs}} = A x^\star + \sigma \varepsilon.
\)
The exact inverse problem setup details are laid out in \cref{sec:repro}.
Because the prior is a GMM and the likelihood is Gaussian, the posterior $p(x \mid y_{\mathrm{obs}})$ is again a (renormalized) GMM with the same number of components. This allows us to (i) draw \emph{exact} posterior samples and (ii) evaluate the \emph{exact posterior score}
\begin{equation}
\label{eq:true_score}
\begin{split}
s^\star(x) &= \nabla_x \log p(x \mid y_{\mathrm{obs}}) = \nabla_x \log p_0(x) + \nabla_x \log p(y_{\mathrm{obs}}\mid x), \\
\nabla_x \log p(y_{\mathrm{obs}}\mid x) &= \frac{1}{\sigma^2}A^\top (y_{\mathrm{obs}}-Ax).
\end{split}
\end{equation}

\subsection*{Dimension coherent noise normalization (inverse SNR coordinate)}
A recurring ambiguity in regime plots is that the \emph{meaning} of ``$\sigma$'' changes with dimension and with the operator $A$ (since $\|Ax\|$ is dimension- and spectrum dependent). To make the horizontal axis comparable across $d$, we sweep a dimensionless, \emph{dimension coherent} inverse SNR parameter
\(
\sigma_{\mathrm{rel}}
\;:=\;
\frac{\sigma}{\sqrt{\mathbb E_{x\sim p_0}\|Ax\|^2}}.
\label{eq:sigma_rel}
\)
The denominator is a signal scale induced by the prior and the forward map. In practice we estimate $\sqrt{\mathbb E\|Ax\|^2}$ once per dimension by Monte Carlo under $p_0$ (and keep it fixed throughout the sweep). We then sweep $\sigma_{\mathrm{rel}} \in [0.025,1.0]$ on a log grid. Smaller $\sigma_{\mathrm{rel}}$ corresponds to higher SNR (sharper likelihood); larger $\sigma_{\mathrm{rel}}$ corresponds to a weaker data term (posterior closer to the prior).

We benchmark our proposed variance-minimizing blend score against a standard SNIS Tweedie baseline. To ensure a direct comparison, both estimators operate on the same reference set and utilize identical SNIS posterior weights. To standardize the generation process across runs, we use fixed reference budget at $N_{\mathrm{ref}}=4,000$ and employ a shared Heun predictor--corrector sampler with matched time grids and sample counts ($N_{\mathrm{gen}}$). This setup isolates the score estimator as the only variable, ensuring that any performance differences are attributable strictly to the blend's statistical properties rather than discrepancies in samplers or compute budgets.

We report the Gaussian kernel MMD between generated and exact posterior samples in log scale, normalized by a per setting \emph{floor}:
\(
\log\!\big(\mathrm{MMD}/\mathrm{floor}\big).
\)
The floor is the MMD between two independent sets of exact posterior samples. It captures the intrinsic finite sample resolution of the metric in that setting; normalizing by it factors out dimension- and sample size effects that would otherwise obscure regime transitions. Metric definitions and kernel choices are deferred to 
\cref{sec:repro}.
\vspace{-1.0em}
\begin{figure}[H]
 \centering
 \includegraphics[width=\linewidth]{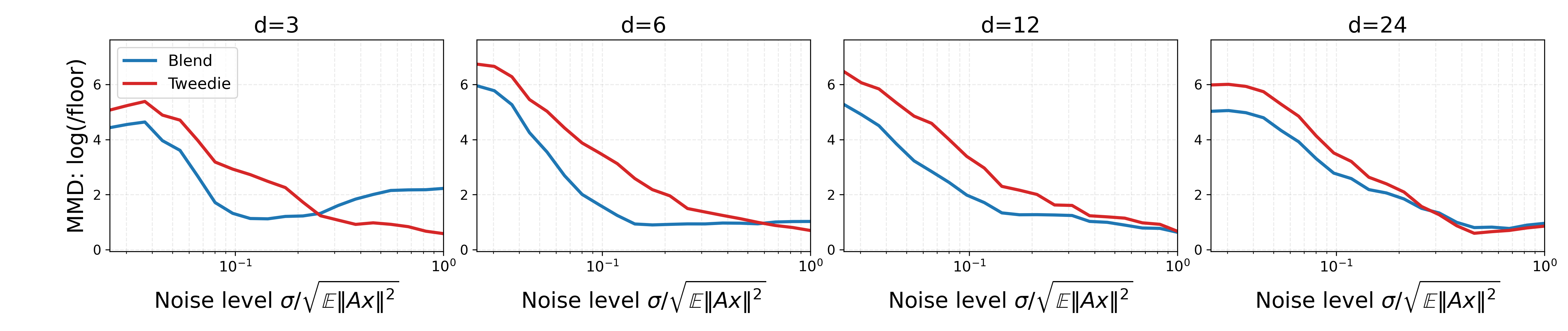}
 \vspace{-1.5em}
 \caption{\textbf{Blend advantage regimes.}
 Sweep of normalized MMD vs.\ inverse SNR ($\sigma_{\mathrm{rel}}$).
 While both estimators suffer weight collapse at high SNR (low $\sigma_{\mathrm{rel}}$, left), \textsc{Blend} (blue) significantly outperforms Tweedie (Red) at intermediate noise levels and dimensions ($d \le 12$).
 The advantage narrows at $d=24$ and in the high-noise limit (right), where the posterior approaches the prior.}
 \label{fig:regime_sweep}
\end{figure}
\vspace{-.5em}
~\cref{fig:regime_sweep} cleanly separates three behaviors that mirror what we observe in the downstream inverse problems:
\begin{enumerate}
\item \textbf{High SNR breakdown (small $\sigma_{\mathrm{rel}}$).}
When the likelihood is sharply concentrated relative to the forward signal scale, both methods suffer (large $\log(\mathrm{MMD}/\mathrm{floor})$). This is the controlled analogue of the ``small noise'' instability observed in the real inverse problems, where importance weighting and finite reference budgets lead to weight degeneracy and high variance.

\item \textbf{Intermediate noise advantage window (moderate $\sigma_{\mathrm{rel}}$).}
For moderate dimensions ($d=3,6,$ and $12$), there is a visibly intermediate band of $\sigma_{\mathrm{rel}}$ where the Blend curve lies below Tweedie on MMD. This is the regime in which the posterior is informative enough that Tweedie only bias/variance is exposed, but not so sharp that all reference based estimators collapse.

\item \textbf{High noise saturation / crossover (large $\sigma_{\mathrm{rel}}$).}
As $\sigma_{\mathrm{rel}}$ increases and the posterior becomes less informative, the curves approach the MMD floor and (in the lowest dimensional case) may cross, indicating that the benefit of blending is concentrated in the intermediate noise window rather than in the prior dominated limit.
\end{enumerate}

The regime characterization in ~\cref{fig:regime_sweep} directly guides the parameter selection for the scientific and imaging inverse problems presented in subsequent sections. We specifically prioritize practical intermediate noise regimes, as this window, situated between the extremes of likelihood dominated collapse and prior dominated equivalence, is where the blended score estimator $\sBLND$ \cref{eq:snis_blended_score} demonstrates the largest gains in our sweep. We also  mitigate the observed sensitivity of reference based estimation to increasing dimension at fixed computational budgets. In our PDE and imaging experiments, we truncate the solution space to a moderate number of principal modes (e.g., 8--24), ensuring the solver operates within the stable capabilities identified in our sweep rather than in a regime dominated by sampling error.

\subsection{Inverse problems}
\label{sec:inverse_problems}

We conclude with inverse problems that probe posterior sampling fidelity in settings with (i) a \emph{white box} prior/likelihood, where the posterior density and score are available (up to normalizing constants), and (ii) a \emph{black box} prior, where the prior is accessible only through samples and must be represented by score proxies. These two regimes naturally support different diagnostics: in the white box setting we can directly evaluate score based discrepancies (e.g., KSD) and density based discrepancies ( e.g., Kullback-Leibler divergence ($\widetilde{\text{KL}}$)). In black box settings we rely on reference posteriors estimated from samples and distributional comparisons (e.g., MMD).

For both problems we work in a reduced coordinate representation $\alpha$ (Karhunen Lo\`eve (KL) coefficients for Navier--Stokes and PCA coefficients for MNIST), and map posterior samples back to the ambient space to evaluate reconstructed fields/images. We compare the Tweedie only posterior sampler ($\sTWD$) against the variance minimized blended posterior sampler ($\sBLND$), formed by posterior tilting as in \cref{sec:bayes-inverse}. Reference posteriors are obtained by Metropolis-Adjusted Langevin Algorithm (MALA) in the white box setting (exact posterior target), and by importance sampling (IS) on a large held out pool in the black box setting. Full metric definitions are deferred to the appendix.
Additional posterior sampling diagnostics, including the inverse heat equation experiment and extra MNIST visualizations,
are deferred to \cref{app:suppres:post,app:suppres:heat}.

To make the two inverse problems as comparable as possible, we report a common core of metrics in both cases: (i) the coefficient-space mean error $\mathrm{RMSE}_\alpha$ (root mean squared error of the posterior mean in coefficient space; see \cref{sec:repro}, Def.~\ref{def:rmse_coeff}); (ii) the ambient-space mean error $\mathrm{RMSE}_{\mathrm{amb}}$ (the same error after mapping samples back to the full field/image; see \cref{sec:repro}, Def.~\ref{def:rmse_ambient}); and (iii) MMD to a reference posterior proxy (MALA for Navier--Stokes and IS for MNIST), which reflects global distributional mismatch to a high-quality baseline. In addition, we report a forward/data-fit error (measuring how well the posterior mean explains the noiseless observation through the forward operator; see \cref{sec:repro}, Def.~\ref{def:fwd_err}). Precise definitions and normalization conventions are again deferred to  \cref{sec:repro}.

\subsubsection{Navier--Stokes inverse problem (white box posterior)}
\label{subsec:ns_inverse}

We evaluate the proposed posterior score estimators in a non linear setting using the 2D Navier--Stokes equations on the torus $\mathbb{T}^2 = [0, 2\pi]^2$. Here we test the upper bound of the advantage regime by increasing the observation noise to $\sigma_{\mathrm{obs}}=0.3$, while maintaining a moderate latent dimension ($d=24$ eigenmodes). The system governs the evolution of the vorticity field $w(x,t)$ according to
\begin{equation*}
\partial_t w + u \cdot \nabla w = \nu \Delta w + f, \qquad
-\Delta \psi = w, \qquad
u = \nabla^\perp \psi,
\end{equation*}
where $\nu$ is the viscosity, $f$ is a forcing term, and $u$ is the incompressible velocity field derived from the streamfunction $\psi$. The parameter of interest is the initial vorticity field $w_0(x)$. We assume a sparse observation model where we measure the velocity field at $25$ spatial locations (yielding $50$ scalar observations) at a final time $T$. The observations $\by \in \R^{50}$ are given by $\by = \mathcal{O}(u(\cdot, T)) + \eta$, with Gaussian noise

\noindent $\eta \sim \mathcal{N}(0, \sigma_{\mathrm{obs}}^2 I)$ where $\sigma_{\mathrm{obs}}=0.3$.

We work in a reduced Karhunen--Lo\`eve parameterization of the initial vorticity,
\begin{equation*}
w_0(x;\alpha)=\sum_{i=1}^{q}\sqrt{\lambda_i}\,\phi_i(x)\,\alpha_i,\qquad \alpha\in\R^{q},\ \ q=24,
\end{equation*}
where $(\lambda_i,\phi_i)$ are the leading eigenpairs of the prior covariance kernel. This yields a Gaussian prior on coefficients
\begin{equation*}
p_0(\alpha)=\mathcal{N}(0,I),\qquad s_0(\alpha) := \nabla_{\alpha}\log p_0(\alpha) = -\alpha,
\end{equation*}
and a Gaussian likelihood induced by the (differentiable) forward operator
\begin{equation*}
F(\alpha) := \mathcal{O}\big(u(\cdot,T;w_0(\cdot;\alpha))\big)\in\R^{50}.
\end{equation*}
The target posterior distribution is defined as
\begin{equation*}
p^{\mathrm{post}}(\alpha\mid \by)\ \propto\ p_0(\alpha)\,p(\by\mid \alpha),
\end{equation*}
with posterior score at $t=0$ given by $s^{\mathrm{post}}_0(\alpha)=s_0(\alpha)+\nabla_{\alpha}\log p(\by\mid \alpha)$, where the likelihood gradient is obtained via the adjoint method (differentiable physics). 

\subsubsection*{Blended posterior score estimation}
Following \cref{sec:bayes-inverse}, we work directly in coefficient space and identify the diffusion state with the KL coefficients, i.e., $x\equiv \alpha$ (so $x_0^{(i)}=\alpha_0^{(i)}$).
We form posterior versions of the Tweedie and TSI estimators by tilting the SNIS logits by $\log \mathcal{L}(\alpha_0^{(i)})$ and using $s^{\mathrm{post}}_0(\alpha_0^{(i)})$ in the posterior correction.
The proposed sampler uses the variance--optimal convex blend of these two estimators ($\sBLND$) along the reverse trajectory, while $\sTWD$ uses the Tweedie term alone.
We also compare against MALA targeting the exact posterior $p^{\mathrm{post}}(\alpha\mid \by)$ as a reference baseline.

\subsubsection*{Experimental setup and metrics}
We compare three sampling strategies: (i) the proposed blended posterior sampler ($\sBLND$), (ii) the Tweedie only posterior sampler ($\sTWD$), and (iii) MALA targeting $p^{\mathrm{post}}(\alpha\mid \by)$ as a reference baseline.
For $\sBLND$ and $\sTWD$, we use $N_{\text{ref}} = 20{,}000$ reference coefficients $\{\alpha_0^{(i)}\}_{i=1}^{N_{\text{ref}}}\sim p_0$,
together with $\{\log \mathcal{L}(\alpha_0^{(i)}),\, s^{\mathrm{post}}_0(\alpha_0^{(i)})\}$, and generate samples with a Heun predictor--corrector integrator using 60 steps.
For MALA, we run chains of 2{,}000 iterations with a burn in of 500 steps.
We report the shared metrics (MMD$\to$MALA, $\mathrm{RMSE}_\alpha$, $\mathrm{RMSE}_{\mathrm{amb}}$, and 
forward/data fit error), and additionally report KSD and $\widetilde{\mathrm{KL}}$ (see \cref{sec:repro}, Def.~\ref{def:kl_tilde}) in this white box setting. 
Here, $\widetilde{\mathrm{KL}}$ acts as a proxy for the forward KL divergence, calculated as the negative differential entropy (estimated via $k$-nearest neighbors) minus the expected unnormalized log-posterior. This metric quantifies distributional discrepancy up to the unknown log-normalizer constant.
\vspace{-.5em}
\begin{table}[H]
 \centering
 \caption{\textbf{Navier--Stokes quantitative results.} Shared metrics (MMD$\to$MALA, mean errors in coefficient/ambient space, and forward/data fit error) are reported alongside KSD and $\widetilde{\mathrm{KL}}$. Arrows indicate the preferred direction: $\downarrow$ means lower values are better (all metrics reported here are minimized), with MALA serving as the reference for MMD$\to$MALA.}
 \label{tab:ns_results}
 \resizebox{\textwidth}{!}{
 \begin{tabular}{lcccccc}
 \toprule
 Method
 & MMD$\to$MALA $\downarrow$
 & $\mathrm{RMSE}_\alpha$ $\downarrow$
 & $\mathrm{RMSE}_{\mathrm{amb}}$ $\downarrow$
 & Fwd Err $\downarrow$
 & KSD $\downarrow$
 & $\widetilde{\mathrm{KL}}$ $\downarrow$ \\
 \midrule
 Tweedie ($\sTWD$)
 & 0.1262
 & 0.5819
 & 0.1201
 & 0.09847
 & 15.80
 & 94.68 \\
 Blend Posterior ($\sBLND$)
 & 0.09022
 & 0.5108
 & 0.1114
 & 0.1029
 & 2.011
 & 50.95 \\
 MALA (Reference)
 & 0.0000
 & 0.4776
 & 0.1012
 & 0.09550
 & 1.774
 & 42.25 \\
 \bottomrule
 \end{tabular}
 }
\end{table}

We compare the Tweedie estimator ($\sTWD$), the Blend estimator ($\sBLND$) and a standard MALA baseline. The results in ~\cref{tab:ns_results} show that $\sBLND$ corrects the approximation error of the pure Tweedie method: it matches the MALA level KSD while maintaining the low computational cost of the Tweedie based method.

 Visually, \cref{fig:ns_hists,fig:ns_recon} tell the same story: $\sBLND $ yields posterior marginals and posterior mean reconstructions and uncertainty maps that are substantially more consistent with those produced by the MALA reference.

\vspace{-1.0em}
\begin{figure}[H]
 \centering
 \includegraphics[width=.75\textwidth]{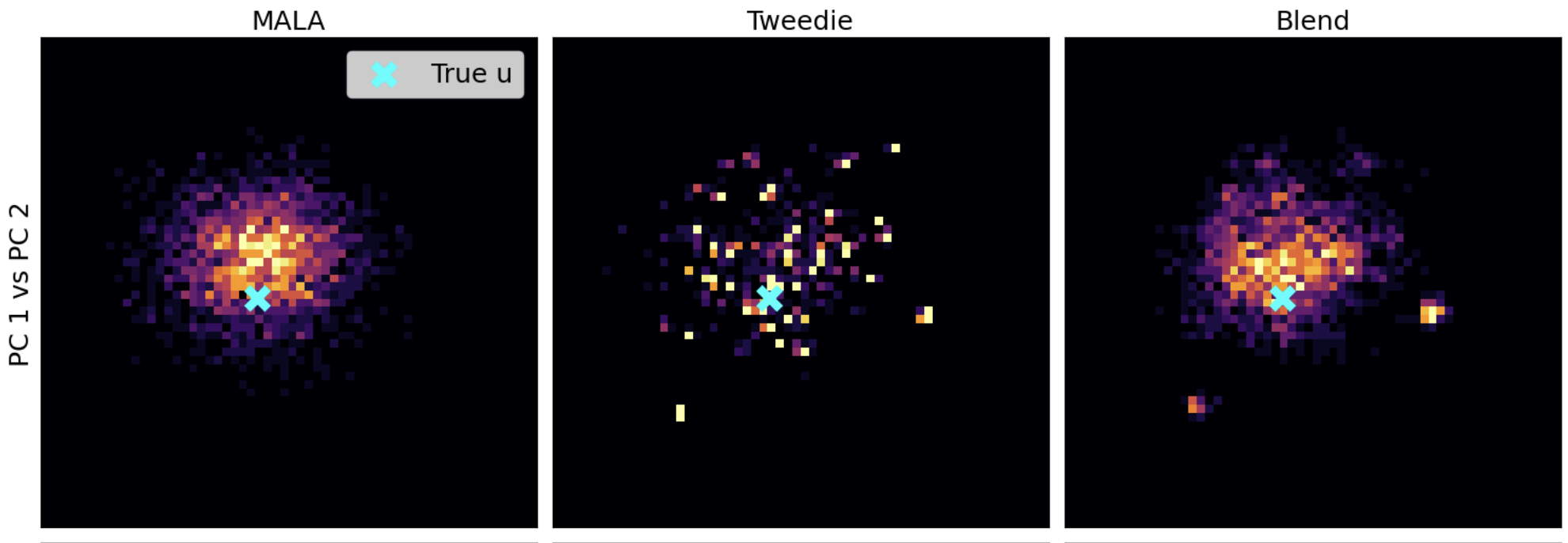}
 \vspace{-.65em}
 \caption{\textbf{Heatmap of Navier--Stokes coefficient posterior histograms.} 
2D marginal histograms of posterior KL coefficients projected onto leading principal components. The Tweedie estimator exhibits severe fragmentation, collapsing probability mass onto a sparse set of reference samples. 
 The \textsc{Blend} estimator uses local geometric information to fill these gaps, restoring the continuous posterior geometry and coverage observed in the MALA ground truth.}
 \label{fig:ns_hists}
\end{figure}

\begin{figure}[H]
 \centering
 \includegraphics[width=.85\textwidth]{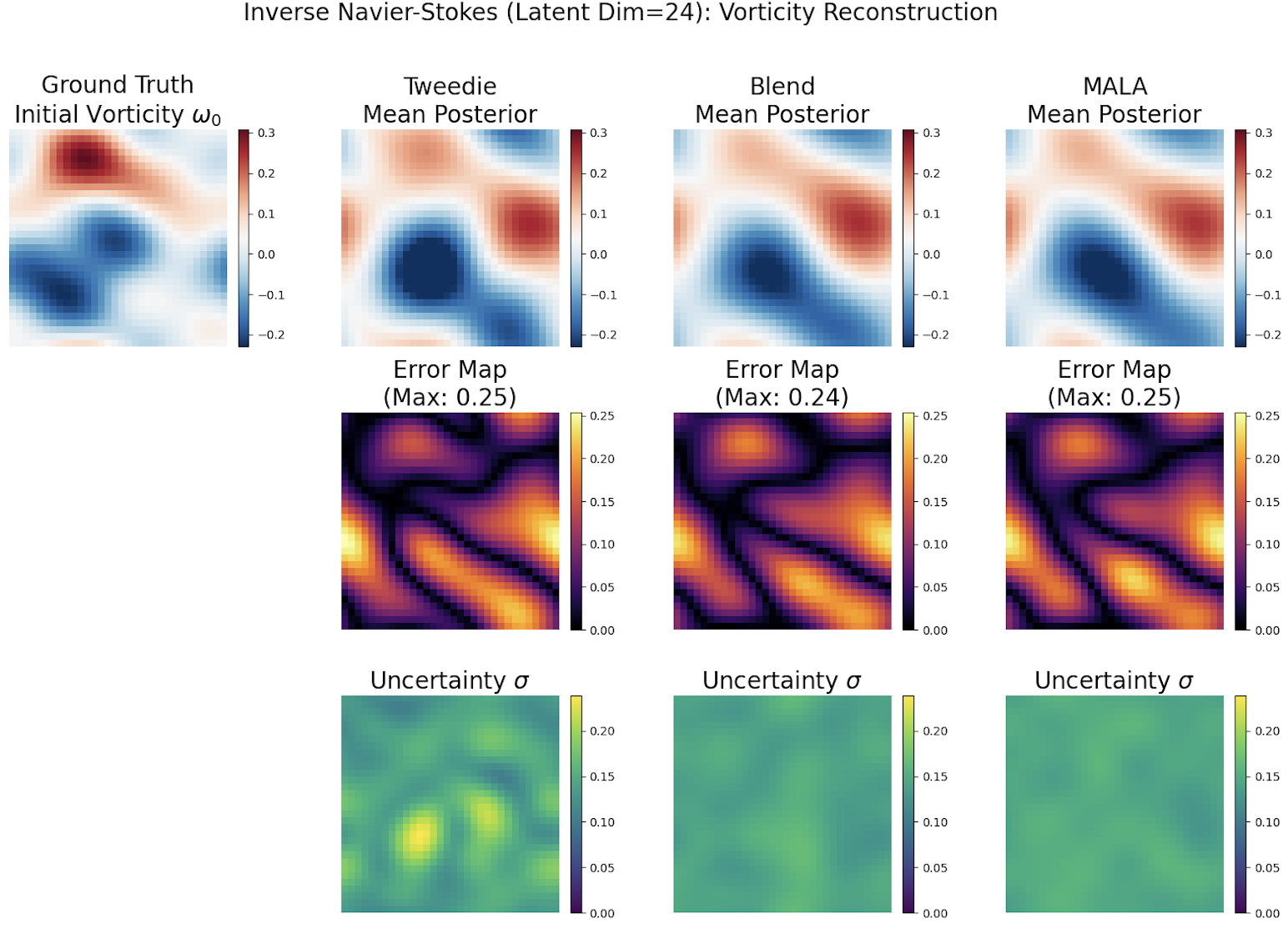}
 \vspace{-.75em}
 \caption{\textbf{Navier--Stokes vorticity field  reconstructions.} Visual comparison of the posterior mean reconstruction of the initial vorticity \(w_0\). The \textsc{Blend} estimator exhibits strong consistency with the MALA ground truth, correctly resolving the intensity and location of vorticity features. 
 In contrast, the pure Tweedie estimator produces a diffuse approximation, failing to resolve the structural uncertainty captured by the reference sampler.}
 \label{fig:ns_recon}
\end{figure}

\subsubsection{MNIST deblurring inverse problem (black box prior)}
\label{subsec:mnist}

We conclude with a practical inverse problem: linear deblurring of MNIST digits. Unlike the Navier--Stokes inverse problem in \cref{subsec:ns_inverse}, where the prior score is known analytically, here the prior distribution $p_0$ is unknown and accessible only through a finite set of samples. We define the problem in a reduced order PCA latent space and estimate the score from these samples using the proposed blended proxy.

\subsubsection*{Setup and latent space}
We utilize the MNIST training set (scaled to $[0,1]$) to compute a Principal Component Analysis (PCA) basis. We retain the top $D=15$ principal directions $U_{15}\in\R^{784\times 15}$ and mean $\mu \in \R^{784}$. Images are projected into this latent space via coefficients $\alpha = U_{15}^\top (x-\mu)\in\R^{15}$.
The prior $p_0(\alpha)$ is the implicit distribution of these MNSIT PCA coefficients. For score estimation, we use a reference set of $N=10{,}000$ such coefficients, $\{\alpha_0^{(i)}\}_{i=1}^{N_{\text{ref}}}\sim p_0$,
and fit the \emph{LR{+}D score proxy} (~\cref{sec:learned_proxy}) directly in this 15-dimensional space.

\subsubsection*{Observation model}
Observations $y_{\text{obs}}$ are generated in the full image space by applying a $9{\times}9$ Gaussian blur kernel $H$ with standard deviation $\sigma_{\text{blur}}{=}2.5$, followed by additive white Gaussian noise $\eta \sim \mathcal{N}(0, \sigma_{\text{obs}}^2 I)$:
\(
y_{\text{obs}} \;=\; H(\mu + U_{15}\alpha) + \eta.
\)

\noindent This configuration targets the intermediate regime identified in ~\cref{sec:regime_study} where variance reduction is most critical.

\subsubsection*{Posterior sampling (implicit prior score + deblurring likelihood)}
To obtain posterior samples for the MNIST deblurring problem, we follow the posterior weight tilting construction in \cref{sec:bayes-inverse}, but with an \emph{implicit} prior in the PCA coefficient space. Concretely, the prior $p_0(\alpha)$ is represented by the reference set of MNIST PCA coefficients, and its score is provided either by the Tweedie estimator $\sTWD$ or by the SNIS plug-in blended estimator $\hat s_{\text{BLEND}}$ defined in \cref{eq:snis_blended_score}, where the TSI term uses the LR{+}D proxy fit to the raw MNIST data (\cref{sec:learned_proxy}).

The observation model induces a linear--Gaussian likelihood
\[
\mathcal{L}(\alpha)
:=p(y_{\text{obs}}\mid \alpha)
=\mathcal{N}\!\big(y_{\text{obs}};\,H(\mu+U_{15}\alpha),\,\sigma_{\text{obs}}^2 I\big),
\]
with log likelihood gradient (in $\alpha$-space)
\[
\nabla_{\alpha}\log \mathcal{L}(\alpha)
=\frac{1}{\sigma_{\text{obs}}^2}\,(HU_{15})^\top\!\Big(y_{\text{obs}}-H(\mu+U_{15}\alpha)\Big).
\]
At sampling time, we perform diffusion and posterior sampling in coefficient space, i.e., the reverse trajectory evolves $\alpha_t\in\R^{15}$.
We form the posterior score estimator exactly as in \cref{sec:bayes-inverse}: we tilt the SNIS weights by $\mathcal{L}(\alpha_0^{(i)})$ and use the resulting posterior version of the score estimator along the reverse trajectory.
Final samples are mapped back to image space by $x=\mu+U_{15}\alpha$.

\subsubsection*{Reference posterior and MALA baseline}
Since the true posterior is intractable, we construct a ``gold standard'' reference distribution using importance sampling (IS) on a large held out pool of prior samples. We also compare against a gradient based MALA baseline targeting an approximate posterior built from a differentiable GMM surrogate prior in latent space (details as in \cref{sec:repro}).

\subsubsection*{Quantitative results}
We report the shared metrics (MMD$\to$IS, $\mathrm{RMSE}_\alpha$,

\noindent $\mathrm{RMSE}_{\mathrm{amb}}$, and forward/data fit error) and, specific to the image setting, PSNR (see Appendix 
\cref{sec:repro}, Def.~
\ref{def:psnr}) and Coverage (see Appendix 
\cref{sec:repro}, Def.~
\ref{def:coverage}). Here Peak Signal-to-Noise Ratio (PSNR) reflects image space reconstruction quality (a monotone transform of pixel space RMSE), while Coverage measures whether generated samples fall within high probability regions of the IS reference posterior. Precise definitions and normalization conventions are deferred to \cref{sec:repro}.

\vspace{-.5em}
\begin{table}[H]
 \centering
 \footnotesize
 \setlength{\tabcolsep}{3pt}
 \caption{\textbf{MNIST deblurring metrics.}
$\uparrow$ means higher value is better (PSNR, Coverage), and $\downarrow$ means lower is better ($\mathrm{RMSE}_\alpha$, $\mathrm{RMSE}_{\mathrm{amb}}$, Fwd Err, and MMD$\to$IS). Blend (proxy) improves posterior fidelity (Coverage, MMD) and image quality (PSNR); shared mean error and forward error metrics are reported for direct comparison to the Navier--Stokes inverse problem.}
 \begin{tabular}{lrrrrrr}
 \toprule
 Method
 & PSNR (dB) $\uparrow$
 & Coverage (\%) $\uparrow$
 & $\mathrm{RMSE}_\alpha$ $\downarrow$
 & $\mathrm{RMSE}_{\mathrm{amb}}$ $\downarrow$
 & Fwd Err $\downarrow$
 & MMD$\to$IS $\downarrow$ \\
 \midrule
 Blend (proxy)
 & 28.02
 & 100.0
 & 0.1550
 & 0.0397
 & 0.1152
 & 0.1086 \\
 Tweedie only
 & 26.98
 & 92.6
 & 0.1843
 & 0.0448
 & 0.1662
 & 0.1876 \\
 MALA GMM
 & 25.99
 & 100.0
 & 0.1898
 & 0.0502
 & 0.1307
 & 0.1324 \\
 \bottomrule
 \end{tabular}
 \label{tab:mnist_combined}
\end{table}
\vspace{-1.0em}
This difference is visually apparent in \cref{fig:mnist_hist}. The blended posterior aligns closely with the IS support, while Tweedie collapses onto a sparse set of reference training samples.

\begin{figure}[H]
 \centering
 \includegraphics[width=.9\linewidth]{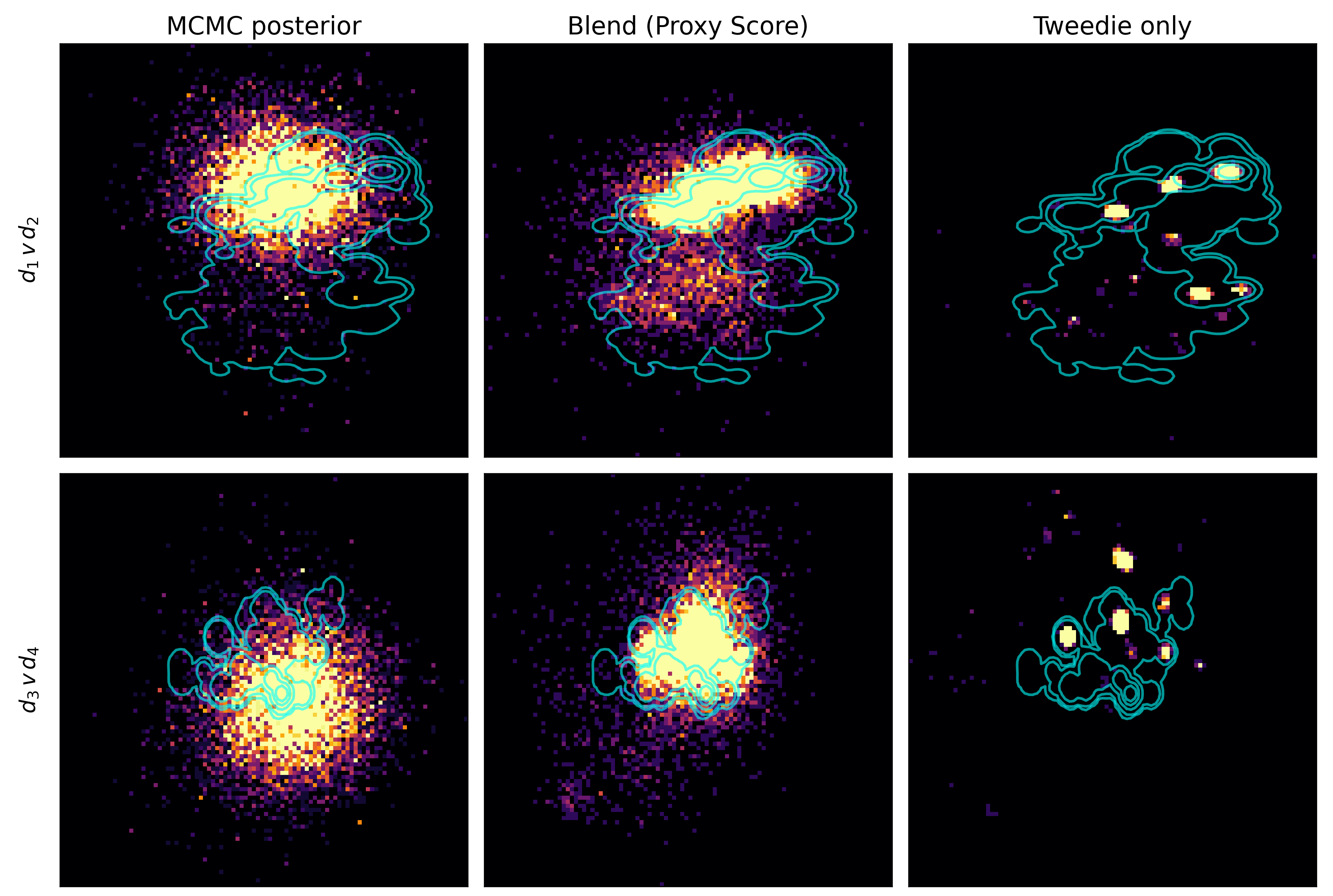}
 \vspace{-.75em}
 \caption{\textbf{Heatmaps of MNIST deblurring posterior PCA coefficients.} 
 Comparison of posterior samples in the PCA plane \((d_1,d_2)\) against the ``gold standard'' IS density contours (cyan). 
 The \textsc{Blend} estimator faithfully fills the non-linear posterior level sets, effectively utilizing the implicit nonparametric prior. 
 This contrasts with Tweedie, which collapses onto a sparse set of reference samples (overfitting), and the parametric GMM assumption used in the MALA baseline, which lacks the expressivity to resolve these complex manifold structures.}
 \label{fig:mnist_hist}
\end{figure}

Additional MNIST posterior sampling visualizations (including a multi panel sample comparison) are provided in
\cref{app:suppres:post}.


\section{Discussion}

Our primary contributions in this work are advancements to the \emph{statistical estimation of score fields} for flow/diffusion models. Rather than proposing a new reverse solver or sampler architecture, we improve the \emph{statistical accuracy of the pointwise score estimates} themselves. We demonstrate that the blended score estimator $\sBLND$ \cref{eq:snis_blended_score} yields a lower-variance estimate by exploiting the \emph{negative correlation} between two complementary score estimators: Tweedie and TSI.
In practice, this blend act as a natural \emph{multiscale decomposition of transport}. At larger diffusion times, the OU kernel is broad, so the Tweedie conditional expectation averages over many reference samples and yields a \emph{stable, coarse grained score field} that helps coordinate \emph{global mass placement}. Conversely, at smaller diffusion times (when ESS remains adequate), the TSI term leverages clean score information (or a local proxy) at nearby references to produce a \emph{locally coherent gradient field} that follows the manifold geometry. The variance-minimizing blend weight $\lambda^*(y,t)$ acts as an automatic gate, smoothly interchanging between these two regimes to minimize estimator error. The result is improved computational efficiency and improved sampling fidelity, driven by lower variance estimation of the underlying geometry.

A practical interpretation of TSI is that it performs \emph{local data augmentation} by propagating neighborhood curvature. While standard estimators view reference samples as Dirac masses, TSI uses local Gaussian score proxies (\textsc{Diag} or \textsc{LR{+}D}; see \cref{sec:learned_proxy}) to model the \emph{shape} of the distribution around each point. This bridges the gap between discrete training samples, making the model behave as thought it had larger Effective Sample Size (ESS) in the local neighborhood of the query.
This helps explain the intermediate regime identified in our regime study (\cref{sec:regime_study}): the estimator shows the largest gains in regimes where the diffusion noise scale is sufficient to overlap these local curvature proxies, allowing the blend to reconstruct the manifold geometry even when the finite reference set is sparse. By explicitly modeling these local gradients, TSI helps to reduce \emph{finite-reference discretizations artifacts} (e.g., sample memorization or local fragmentation) that arise when the score field is reconstructed from a discrete reference set.

Our investigation into inverse problems reveals that the Tweedie estimator's tendency to collapse is exacerbated in posterior sampling. As shown in the regime sweep in ~\cref{fig:regime_sweep}, this performance gap widens as the observation noise decreases (increasing likelihood sharpness) and as the latent dimension increases. This pathology arises from a ``Double Jeopardy'' variance trap that the blended score estimator $\sBLND$ \cref{eq:snis_blended_score} escapes.
At large diffusion times, the dominant failure mode is weight collapse: the self normalized importance sampling (SNIS) weights for the posterior differ from the prior weights by the likelihood term,
\(
\tilde{w}_i^{\text{post}} \propto p_{t|0}(y \mid x_0^{(i)}) \cdot \mathcal{L}(x_0^{(i)}).
\)

\noindent In high dimensional inverse problems, $\mathcal{L}(x)$ concentrates mass on a thin manifold. When using a fixed reference set from a diffuse prior, the ESS degrades rapidly. In the limit ($\text{ESS} \to 1$), the estimator becomes dominated by the single reference particle $x_0^{(k)}$ maximizing the likelihood kernel product. Consequently, the score degenerates to $\sTWD \approx \sigma_t^{-2}(e^{-t}x_0^{(k)} - y)$, acting as a linear restoring force toward a single training point rather than interpolating the posterior manifold (the ``fragmented'' memorization seen in 
\cref{fig:mnist_hist}).

A second failure mode appears as $t\to 0$. Even though the likelihood gradient becomes sharp, Tweedie’s prior term scales like $\sigma_t^{-2}\approx t^{-1}$, so estimator noise dominates precisely in the small time regime where the likelihood would otherwise help. While the likelihood signal $\nabla \log \mathcal{L}$ becomes sharp at small $t$, the Tweedie estimator is statistically incapable of resolving it because the noise in the prior estimation dominates the signal. Thus, Tweedie fails at large $t$ (due to weight collapse) \emph{and} at small $t$ (due to variance explosion).
TSI addresses this failure mode by carrying the likelihood gradient inside the transported estimate:
\(
\sTSI^{\text{post}}(y,t) \approx e^t\!\left(\hat{s}_0 + \nabla_{x_0}\log \mathcal{L}(\hat{x}_0)\right).
\)

\noindent Because the likelihood term is deterministic and the prefactor remains well behaved as $t\to 0$, this contribution stays low variance in the small time regime. The blended score estimator $\sBLND$ \cref{eq:snis_blended_score} then uses TSI precisely where Tweedie becomes unstable, so the likelihood correction is injected at the times when it can be resolved statistically.

Beyond nonparametric sampling, a primary utility of this framework is generating an \emph{augmented statistical signal} for training high quality diffusion models. In both SciML settings (where the exact score is computable but expensive) and pure ML settings (where the score is learned), the blended construction underlying $\sBLND$ \cref{eq:snis_blended_score} can be used as a \emph{low-variance teacher}. 
As detailed in the ``Critic--Gate'' analysis (\cref{app:critic-gate}), we train a gate $g(y,t)$ to mix two per particle unbiased signals (TSI and Tweedie) so as to reduce posterior variance, while a critic $q(y,t)$ amortizes the posterior mean of this gated signal into a single parametric score model. This yields a stronger supervision signal than standard denoising losses, because the student is trained against a variance reduced target rather than high variance per sample noise.

\newpage
We conclude with a discussion of the main practical limitations to our approach. In the most challenging inverse regimes, sharply concentrated likelihoods or high effective dimension, finite reference reweighting can degenerate and the estimator becomes unstable. A second limitation is geometric: if the forward noise scale fails to overlap the local curvature proxies, local linear/quadratic corrections cannot reliably bridge gaps between sparse references. These issues motivate the scaling directions summarized in the conclusion.

\section{Conclusion and Future Work}
We reframed score learning as a \emph{statistical estimation} problem at a queried $(y,t)$ and introduced a blended score estimator that combines two complementary signals. To this end, the key piece of machinery exploited in this work is the Target Score Identity (TSI) identity, which transports score information across time through the forward transition kernel. Using the Ornstein Uhlenbeck (OU) flow as a canonical worked example, we constructed a nonparametric TSI estimator, and paired it with the classical Tweedie estimator to produce blended a score estimator with reduced variance. We proved \emph{exact} anticorrelation between the Monte Carlo errors of the nonparametric TSI and Tweedie score estimators in the linear--Gaussian case, and we show that this correlation remains negative for sufficiently small diffusion times (for large enough reference sets) under mild regularity conditions on $p_0$, yielding a \emph{state and time dependent variance-minimizing convex blend} with closed form optimal weight $\lambda^{*}(y,t)$.
 We showed SNIS plug-in estimates provide the quantities needed to compute $\lambda^{*}$, while local Gaussian score proxies (\textsc{Diag}/\textsc{LR{+}D}; \cref{sec:learned_proxy}) supply stable curvature information when ground truth $s_0$ is unavailable. The same blended score estimator extends to posterior inference by a one line likelihood informed reweighting of the SNIS weights. We use the OU process as a canonical worked example, but \cref{app:affine-CSE} derives the same TSI identity in closed form for general affine diffusions (including VP and VE); our claims and constructions are formulated at this level of generality.

Our future roadmap keeps the estimator centric viewpoint but moves toward neural implementations along three parallel tracks: distillation, scaling, and statistical robustness. First, we want a \emph{neural distillation} of the blended score estimator. Concretely, we treat the TSI and Tweedie estimators as two unbiased (but differently noisy) training signals for the same target score, and we train a network to (i) predict the score and (ii) predict the blending weight by minimizing an MSE criterion that reflects the per input variance tradeoff. Rather than fixing a hand designed blend schedule or supervising with a pre averaged estimator, we learn the blend as part of the score learning objective so that the network can adapt the mixture across diffusion time and across input locations. \cref{app:critic-gate} contains a prototype of this joint learning procedure, which we refer to as the \emph{Critic--Gate} method.

To scale these benefits to high dimensional image benchmarks, we will move beyond Gaussian proxies by developing \emph{curvature aware embeddings}. These latent spaces, such as VAEs that expose local covariance, will allow the TSI transport machinery to operate at low computational cost in that augmented geometry. Finally, to address the concentration barriers inherent in standard importance sampling, we will incorporate \emph{stability oriented sampling strategies} such as tempering, and sequential Monte Carlo. These remedies aim to stabilize the plug-in variance estimates and the mixing weight $\lambda^*(y,t)$ for concentrated posteriors. Across these directions, the objective remains the same: to deliver a \emph{lower variance local score estimate}, that downstream samplers and neural students can exploit for higher fidelity sampling given a fixed compute budget.

\newpage
\appendix

\section{Derivation of the TSI Identity for Linear/Affine SDEs}
\label{app:affine-CSE}

For completeness and to fix notation for the affine diffusion cases (OU, VP, VE), we provide a self-contained derivation of the Target Score Identity (\cref{eq:CSEidentity}) below. While equivalent to the results established in \cite{debortoli2024tsm}, our derivation works directly with the transition kernel and its dependence on the initial condition $x_0$, which is central to our variance analysis.

To begin, we consider the time inhomogeneous affine SDE on $\R^d$ defined by
\[
dX_t = A(t) X_t dt + b(t) dt + G(t) dW_t, \qquad X_0 \sim p_0,
\]
where $A(t)\in\R^{d\times d}$, $b(t)\in\R^d$, and $G(t)\in\R^{d\times r}$ are measurable and locally bounded functions, and $W_t$ denotes an $r$-dimensional standard Brownian motion.

Let $\Phi(t,s)\in\R^{d\times d}$ be the \emph{fundamental matrix} associated with the linear ODE $\dot{Z}(t)=A(t)Z(t)$. It is the unique matrix function satisfying
\[
\partial_t \Phi(t,s) = A(t)\Phi(t,s),\qquad \Phi(s,s)=I_d.
\]
In the time homogeneous case where $A(t) \equiv A$, this matrix simplifies to $\Phi(t,s) = e^{A(t-s)}$.

It is well known that the solution $X_t$ constitutes a Gaussian process \cite{Oksendal98,Gardiner04}. The transition kernel $p_{t|0}(y \mid x)$ takes the form 
 $\mc{N}(y; \Phi(t,0)x + m(t), \Gamma(t))$, where the mean offset $m(t)$ and covariance $\Gamma(t)$ are defined as
\[
m(t) := \int_{0}^{t} \Phi(t,\tau) b(\tau) d\tau, \qquad \Gamma(t) := \int_{0}^{t} \Phi(t,\tau) G(\tau)G(\tau)^{\top} \Phi(t,\tau)^{\top} d\tau.
\]
We denote the score of the time-$t$ marginal density $p_t(y)$ by $s(y,t):=\nabla_y \log p_t(y)$ and the initial score by $s_0(x):=\nabla_x \log p_0(x)$. We assume that for each $t>0$, the transition is nondegenerate (i.e., $\Gamma(t)$ is positive definite) and that the boundary terms vanish during integration by parts. Given these prerequisite definitions we can state and prove the TSI for general affine SDEs.
The following theorem restates the TSI for affine diffusions in our notation (cf.\ \cite{debortoli2024tsm}).

\begin{theorem}[TSI for Linear/Affine SDEs]
Consider any affine SDE satisfying the conditions outlined above. Then, for every $t>0$ and $y \in \R^d$, the score function satisfies
\begin{equation}
\label{eq:CSE_general}
\boxed{ s(y,t) = \Phi(t,0)^{-\top} \E_{x_0 \sim p_{t|0}(\cdot|y)}\LRs{s_0(x_0)}, }
\end{equation}
where $p_{t|0}(x_0 \mid y) = p(x_0 \mid X_t=y)$ denotes the posterior distribution of the initial data given the noisy observation $y$.
\end{theorem}

\begin{proof}
The Gaussian transition kernel is given by
\[
p_{t|0}(y\mid x) \propto \exp\LRp{-\frac{1}{2} \nor{y - (\Phi(t,0)x+m(t))}_{\Gamma(t)^{-1}}^2}.
\]
Taking gradients with respect to $y$ and $x$ yields the following cross derivative identity:
\begin{equation}
\label{eq:cross_deriv}
\nabla_y p_{t|0}(y\mid x) = -\Phi(t,0)^{-\top} \nabla_x p_{t|0}(y\mid x)
\end{equation}
By definition, the score is $s(y,t) = \frac{\nabla_y p_t(y)}{p_t(y)}$. Differentiating the marginal density $p_t(y) = \int p_{t|0}(y\mid x) p_0(x) dx$ under the integral sign and applying the identity \cref{eq:cross_deriv}, we obtain
\[
\nabla_y p_t(y) = \int \nabla_y p_{t|0}(y\mid x) p_0(x) dx = -\Phi(t,0)^{-\top} \int \nabla_x p_{t|0}(y\mid x) p_0(x) dx.
\]
Integrating the right hand side by parts with respect to $x$ gives
\[
\int \nabla_x p_{t|0}(y\mid x) p_0(x) dx = - \int p_{t|0}(y\mid x) \nabla_x p_0(x) dx = - \int p_{t|0}(y\mid x) p_0(x) s_0(x) dx.
\]
Substituting this result back into the expression for $\nabla_y p_t(y)$, we have
\[
\nabla_y p_t(y) = \Phi(t,0)^{-\top} p_{t|0}(y\mid x) p_0(x) s_0(x) dx = \Phi(t,0)^{-\top} p_t(y) \E_{x_0 \sim p_{t|0}(\cdot|y)}\LRs{s_0(x_0)}.
\]
The proof is concluded by dividing both sides by $p_t(y)$.
\end{proof}
 The standard OU process (\cref{eq:OUprocess}) corresponds to $A(t) \equiv -I_d$. In this case, the fundamental matrix is $\Phi(t,0) = e^{-t}I_d$. Substituting this into the general identity (\cref{eq:CSE_general}) gives:
\[
s(y,t) = (e^{-t}I_d)^{-\top} \E_{x_0 \sim p_{t|0}(\cdot|y)}[s_0(x_0)] = e^{t} \E_{x_0 \sim p_{t|0}(\cdot|y)}[s_0(x_0)],
\]
which is exactly the identity presented in the main text in \cref{eq:CSEidentity}.

This generalized TSI applies to all common linear SDEs used in generative modeling. We consider the main canonical examples below

\subsection*{I. Variance Preserving (VP) SDE} For $dX_t = -\frac{1}{2}\beta(t)X_t dt + \sqrt{\beta(t)}dW_t$, we have $\Phi(t,0)=\alpha(t)I$ where $\alpha(t) = \exp(-\frac{1}{2}\int_0^t \beta(u)du)$.
The TSI is:
\[
s(y,t) = \alpha(t)^{-1} \E_{x_0 \sim p_{t|0}(\cdot|y)}[s_0(x_0)].
\]

\subsection*{II. Variance Exploding (VE) SDE} For $dX_t = g(t)dW_t$, we have $\Phi(t,0)=I$. The TSI is:
\[
s(y,t) = \E_{x_0 \sim p_{t|0}(\cdot|y)}[s_0(x_0)].
\]

\subsection*{III. Anisotropic OU / Whitening SDE} For $dX_t = A X_t dt + G dW_t$ with constant matrices $A$ and $G$, we have $\Phi(t,0)=e^{At}$. The TSI is:
\[
s(y,t) = e^{-A^{\top} t} \E_{x_0 \sim p_{t|0}(\cdot|y)}[s_0(x_0)].
\]

The relationship between the Tweedie perspective in \cref{sect:tweedie} and the TSI in \cref{eq:CSE_general} is governed by the \emph{Gradient--Semigroup Commutation (GSC)} principle \cite{pardoux2008markov}.
Let $P_t$ denote the forward evolution (pushforward) operator acting on the initial density $p_0$ via the affine transition kernel $p_{t|0}$, i.e.,
\[
(P_t p_0)(y)=p_t(y)=\int p_{t|0}(y\mid x)\,p_0(x)\,dx,
\quad
p_{t|0}(y\mid x)=\mc{N}\!\big(y;\,\Phi(t,0)x+m(t),\,\Gamma(t)\big).
\]
For affine diffusions, the Gaussian form implies the cross derivative identity
\[
\nabla_y p_{t|0}(y\mid x)=-\Phi(t,0)^{-\top}\nabla_x p_{t|0}(y\mid x),
\]
so differentiation under the integral sign and integration by parts yield the commutation rule
\begin{equation}
\label{eq:GSC_affine}
\begin{split}
\nabla_y (P_t p_0)(y) &= \Phi(t,0)^{-\top}\,(P_t \nabla p_0)(y), \\
(P_t \nabla p_0)(y) &:= \int p_{t|0}(y\mid x)\,\nabla_x p_0(x)\,dx.
\end{split}
\end{equation}
Dividing \cref{eq:GSC_affine} by $p_t(y)=(P_t p_0)(y)$ gives the corresponding identity for the score :
\[
\nabla_y \log (P_t p_0)(y)
\;=\;
\Phi(t,0)^{-\top}\,
\E_{x_0\sim p_{t|0}(\cdot\mid y)}\!\Big[\nabla_{x_0}\log p_0(x_0)\Big]
\]
In words, smoothing the density and then taking a gradient in $y$ is equivalent to taking the initial gradient field in $x$ and then smoothing it against the posterior, with the linear prefactor $\Phi(t,0)^{-\top}$ determined by the drift. For the standard OU choice $A(t)\equiv -I_d$, we have $\Phi(t,0)=e^{-t}I_d$ and the prefactor reduces to $e^t$, recovering the OU specific statement used in the main text.

\section{Proofs}
\label{sect:proofs}
\subsection{Auxiliary results for the proof of \cref{thm:negcorrGeneral}}
The following result states that the conditional distribution $p_{t|0} := p_{t|0}(x\mid y)$ is strongly log concave. This is obvious for all $t$ if $m(y) \ge 0$, and thus we focus on the case when $m(y) < 0$.
\begin{lemma}[Strong log concavity of $p_{t|0}(x\mid y)$ for small time when $m(y) < 0$]
\label{lem:logconcavity}
If $t <  \frac{1}{2}\log\left(1 - \frac{1}{m(y)}\right)$, then $p_{t|0}(x_0\mid y)$ is strongly log concave, meaning that:
\[
-\nabla^2_{x_0}\log p_{t|0}(x_0\mid y) \succeq \kappa\LRp{y} I \succ 0, \text{ where } \kappa\LRp{y}:= m\LRp{y} + \frac{e^{-2t}}{1-e^{-2t}},
\]
and 
\[
 \Sigma := \Cov_{p_{t|0}}\LRp{X} \preceq \frac{1}{\kappa\LRp{y}}I.
\]
\end{lemma}
\begin{proof}
From \cref{eq:OU_posterior} we have
\begin{multline*}
    -\nabla^2_{x_0}\log p_{t|0}(x_0\mid y) = -\nabla^2_{x_0}[\log p_0(x_0)] + \frac{e^{-2t}}{1-e^{-2t}}I \\
    \succeq \LRp{m\LRp{y} + \frac{e^{-2t}}{1-e^{-2t}}}I \succeq  \kappa\LRp{y} I  \succ 0,
\end{multline*}
for all $t < \frac{1}{2}\log(1 - 1/m(y))$. The second assertion is obvious by the Brascamp Lieb inequality \cite{BrascampLieb1976,BrascampLieb1975}.
\end{proof}
Now, define the true conditional mean as $\mu := \E\LRs{X_0|X_t = y}$. Using a first order Taylor expansion of the score $s_0$ around $\mu$ we have
\[
s_0(x) = s_0(\mu) + \nabla s_0(\mu)(x - \mu) + f(x),
\]
where 
\begin{equation} 
\nor{f\LRp{x}} \le c \nor{x -\mu}^2, \text{ since } \nor{\nabla^3_{x}[\log p_0(x)]}_{op} \le c.
\label{eq:fquadraticGrowth}
\end{equation}
As a result, the exact TSI score \cref{eq:nonParametricCSE} is now given as
\[
s_C\LRp{y,t} = e^t\mathbb{E}_{p_{t|0}}[s_0(X_0)] = e^t[s_0(\mu) + \mathbb{E}_{p_{t|0}}[f\LRp{X_0}]],
\]
where, by \cref{lem:logconcavity},
\[
|\mathbb{E}_{p_{t|0}}[f(X_0)]| \leq \frac{c}{2}\mathbb{E}_{p_{t|0}}[\|X_0 - \mu\|^2] = \frac{c}{2}\tr(\Sigma) \leq \frac{c d}{2\kappa(y)},
\]
which is small for small time t as $\kappa \approx (2t)^{-1}$.
The SNIS estimator of $s_C$ is 
\[
\hat{s}_C = e^t\LRs{s_0(\mu) + \nabla s_0^T(\mu)(\hat{\mu} - \mu) + \hat{f}}, \text{ where } \hat{\mu} = \sum_i \tilde{w}_i X_0^{i}, \text{ and }
\hat{f} := \sum_i \tilde{w}_i f\LRp{X_0^{i}}.
\]
Consequently, the deviation of the SNIS TSI score is 
\[
\varepsilon_{\TSI} = \hat{s}_C - s_C = e^t\LRs{\nabla s_0^T(\mu)\varepsilon_\mu + \varepsilon_{f}}, \text{ where } \varepsilon_\mu = \hat{\mu} - \mu, \text{ and } \varepsilon_{f} = \hat{f}-\mathbb{E}_{p_{t|0}}[f\LRp{X_0}].
\]
Similarly, the deviation of SNIS estimation of Tweedie is given by
\[
\varepsilon_T = \hat{s}_T - s_T =   \frac{e^{-t}}{1-e^{-2t}}\varepsilon_\mu.
\]
The correlation between TSI and Tweedie is thus
\begin{equation}
\label{eq:CandTcorrelation}
\E_{p_0}\LRs{\varepsilon_{\TSI}^T\varepsilon_T} = -\underbrace{\frac{1}{1-e^{-2t}}\mathbb{E}_{p_0}\LRs{\epsilon_\mu^T\Sigma_{\text{eff}}^{-1}\epsilon_\mu}}_{D} + \underbrace{\frac{1}{1-e^{-2t}}\mathbb{E}_{p_0}[\varepsilon^T_f\epsilon_\mu]}_{E}.
\end{equation}

\begin{lemma}[Bounding the dominant term $D$]
\label{lem:D}
 There holds: 
 \[ D  = \frac{1}{1-e^{-2t}}\tr\LRp{\Sigma_{\text{eff}}^{-1}\mathrm{Cov}_{p_0}\!\left(\widehat{\mu}\right)}\gtrsim \frac{1}{1-e^{-2t}}\frac{w_{\text{min}}}{N_{\text{ref}}} \lambda_{\text{min}}\LRp{\Sigma_{\text{eff}}^{-1}} \nor{\Sigma}_{op}.
    \]
\end{lemma}
\begin{proof}
Using the standard delta method and the central limit for SNIS \cite{vaart1998asymptotic,kong1994sequential,owen2013monte,liu2001combined} gives
\begin{subequations}
\label{eq:CovXAsympN}
\begin{equation}
\label{eq:CovXAsympNCLT}
\mathrm{Cov}_{p_0}\!\left(\widehat{\mu}\right)
=
\frac{1}{N_{\text{ref}}}\,\Omega_\mu + o\!\left(\frac{1}{N_{\text{ref}}}\right)
\quad\text{as } N_{\text{ref}}\to\infty,
\end{equation}
where $\Omega_\mu 
= 
\mathbb{E}_{p_{t|0}}\!\left[\,w(X_0)\,(X_0-\mu)(X_0-\mu)^{\!\top}\right].$ 
 Since the importance weights are 
 
 \noindent bounded, we obtain
\begin{equation*}
\|\Omega_\mu\|_{op} \ge w_{\min}\;\nor{\Sigma}_{op}.
\end{equation*}
\end{subequations}
Using \cref{eq:CovXAsympN} we have
\begin{multline*}
\mathbb{E}_{p_0}\LRs{\epsilon_\mu^T\Sigma_{\text{eff}}^{-1}\epsilon_\mu} =  \tr\LRp{\Sigma_{\text{eff}}^{-1}\mathrm{Cov}_{p_0}\!\left(\widehat{\mu}\right)} \ge \lambda_{\text{min}}\LRp{\Sigma_{\text{eff}}^{-1}} \tr\LRp{ \mathrm{Cov}_{p_0}\!\left(\widehat{\mu}\right)} \\ \gtrsim
\lambda_{\text{min}}\LRp{\Sigma_{\text{eff}}^{-1}}\frac{\tr\LRp{\Omega_\mu}}{N_{\text{ref}}}  \gtrsim \frac{\lambda_{\text{min}}\LRp{\Sigma_{\text{eff}}^{-1}}}{N_{\text{ref}}}\nor{\Omega}_{op} \gtrsim \frac{w_{\text{min}}}{N_{\text{ref}}} \lambda_{\text{min}}\LRp{\Sigma_{\text{eff}}^{-1}} \nor{\Sigma}_{op}.
\end{multline*}
\end{proof}

\begin{lemma}[Bounding the cross term $E$]
\label{lem:E}
 There holds:
 \[
    E \le \frac{1}{1-e^{-2t}} \frac{w_{\max}}{N_{\text{ref}}}\; \frac{cd}{\kappa} \sqrt{\tr\LRp{\Sigma}}.
    \]
\end{lemma}
\begin{proof}
Similar to \cref{eq:CovXAsympNCLT}, we have
\[
\mathrm{Cov}_{p_0}\!\left(\widehat{f}\right)
=
\frac{1}{N_{\text{ref}}}\,\Omega_f + o\!\left(\frac{1}{N_{\text{ref}}}\right)
\quad\text{as } N_{\text{ref}}\to\infty,
\]
where $\Omega_f
= 
\mathbb{E}_{p_{t|0}}\!\left[\,w(X_0)\,\LRp{f\LRp{X_0}-\mathbb{E}_{p_{t|0}}[f\LRp{X_0}]}(f\LRp{X_0}-\mathbb{E}_{p_{t|0}}[f\LRp{X_0}])^{\!\top}\right].$
 
 \noindent By the Cauchy Schwarz inequality we have
\begin{multline*}
    \E_{p_0}\LRs{\varepsilon^T_f\epsilon_\mu} 
    \le \sqrt{\tr\LRp{\Cov_{p_0}\LRp{\hat{f}}}} \sqrt{\tr\LRp{\Cov_{p_0}\LRp{\hat{\mu}}}} \\
    \lesssim \frac{w_{\max}}{N_{\text{ref}}}\; 
    \sqrt{\E_{p_{t|0}}\LRs{\nor{f\LRp{X_0}-\mathbb{E}_{p_{t|0}}[f\LRp{X_0}]}^2}}\sqrt{\tr\LRp{\Sigma}}.
\end{multline*}
Using Jensen's inequality and \cref{eq:fquadraticGrowth} gives
\[
\nor{\E_{p_{t|0}}\LRs{f\LRp{X_0}}}
 \le \E_{p_{t|0}}\nor{f\LRp{X_0}} 
\le c\,\E_{p_{t|0}}\|X_0-\mu\|^2 = c\,\tr\LRp{\Cov_{p_{t|0}}(X_0)} \le \frac{c d}{\kappa\LRp{y}}.
\]
On the other hand, since $p_{t|0}$ is strongly log concave (see \cref{lem:logconcavity}), $\langle u,X-m\rangle$ are sub Gaussian with
 variance proxy $\kappa^{-1}$ for any unit vector $u$ \cite{bakry1985diffusions}. Standard moment estimates for sub Gaussian distribution \cite{ledoux2001concentration,vershynin2018high,wainwright2019high} then give 
 \[
      \mathbb{E}_{p_{t|0}}\|X_0-\mu\|^4
      \;\lesssim \;
      \frac{d^2}{\kappa^2},
      \]
thus
\[
\E_{p_{t|0}}\nor{f\LRp{X_0}}^2 \le c^2 \E_{p_{t|0}}\nor{X_0-\mu}^4 
\lesssim \frac{c^2d^2}{\kappa^2}
\]
Next using triangle inequality we have
\[
\E_{p_{t|0}}\nor{f\LRp{X_0}-\mathbb{E}_{p_{t|0}}[f\LRp{X_0}]}^2
\le 2 \E_{p_{t|0}}\nor{f(X_0)}^2 + 2\nor{\E_{p_{t|0}}\LRs{f\LRp{X_0}}}^2 \lesssim \frac{c^2d^2}{\kappa^2}.
\]
We conclude
\[
E \lesssim \frac{1}{1-e^{-2t}} \frac{w_{\max}}{N_{\text{ref}}}\; \frac{cd}{\kappa} \sqrt{\tr\LRp{\Sigma}} \le \frac{1}{1-e^{-2t}} \frac{w_{\max}}{N_{\text{ref}}}\; \frac{cd^{3/2}}{\kappa^{3/2}}.
\]
\end{proof}

\subsection{Proof of \cref{prop:negcorr-gaussian}}
\label{app:proof:negcorr-gaussian}
\begin{proof}
Substituting the exact Gaussian score $s_0(x) = -\Sigma^{-1}(x - \mu_0)$ into the SNIS estimator in \cref{eq:nonParametricCSE} and \cref{eq:nonParametricTweedie} yields
\[
\sTSI = -e^t\Sigma^{-1} [\hat{\mu}_{\text{SNIS}} - \mu_0], \quad \text{ and }
\sTWD = -\frac{1}{1-e^{-2t}}\left[y - e^{-t}\hat{\mu}_{\text{SNIS}}\right],
\]
and all the assertions follows
\end{proof}

\subsection{Proof of \cref{prop:OptimalVariance}}
\label{app:proof:OptimalVariance}
\begin{proof}
By setting $\frac{\partial J}{\partial \lambda} = 0$ we obtain
\begin{equation}
\lambda^{*}=\frac{\sigma_C^2-\rho\,\sigma_T\sigma_C}{\sigma_T^2+\sigma_C^2-2\rho\,\sigma_T\sigma_C},
\text{ and }
J(\lambda^{*})=\frac{\sigma_T^2\,\sigma_C^2\,(1-\rho^2)}{\sigma_T^2+\sigma_C^2-2\rho\,\sigma_T\sigma_C}.
\end{equation}
Since $\sigma_T^2+\sigma_C^2-2\rho\,\sigma_T\sigma_C > 0$, both assertions can be verified by direct algebraic manipulations.
\end{proof}

\section{Details of the local Gaussian score proxy}
\label{app:score_proxy_details}

This Appendix records implementation details for the local Gaussian score proxies
(\textsc{Diag} and \textsc{LR{+}D}) and the optional $k$-mix recomputation step.
These procedures are standard and are included only to make our experimental setup
reproducible.

\subsection{Anchor fitting via weighted $k$NN}
\label{app:proxy_knn}

Let $X=\{x_0^i\}_{i=1}^{N_{\mathrm{ref}}}\subset\mathbb{R}^d$ be reference samples from $p_0$.
For each anchor $x_0^i$, we let $\mc{N}_k(i)$ denote the indices of its $k$ nearest neighbors
under the ambient Euclidean metric. We set an adaptive bandwidth by
\begin{equation*}
  h_i^2 \;:=\; \max_{j\in\mc{N}_k(i)} \|x_0^i-x_0^j\|_2^2.
\end{equation*}
We define unnormalized kernel weights $\bar w_{ij}$ and their normalized versions $w_{ij}$ by
\begin{equation*}
  \bar w_{ij}
  \;:=\;
  \exp\!\Big(-\frac{\|x_0^i-x_0^j\|_2^2}{2h_i^2}\Big),
  \qquad
  w_{ij}
  \;:=\;
  \frac{\bar w_{ij}}{\sum_{\ell\in\mc{N}_k(i)}\bar w_{i\ell}},
  \qquad
  j\in\mc{N}_k(i).
\end{equation*}
We then compute the locally weighted mean
\begin{equation*}
  \mu_i \;:=\; \sum_{j\in\mc{N}_k(i)} w_{ij}\,x_0^j.
\end{equation*}
Given a positive definite covariance model $\Sigma_i\succ 0$, the Gaussian score proxy
at the anchor is $\Sigma_i^{-1}(\mu_i-x_0^i)$, as defined in
\cref{eq:proxy_diag,eq:proxy_lrd}.

\begin{algorithm}[H]
\caption{Local Gaussian proxy at anchors (\textsc{Diag} or \textsc{LR{+}D})}
\label{alg:local_score_estimation}
\begin{algorithmic}[1]
\State \textbf{Input:} $X=\{x_0^i\}_{i=1}^{N_{\mathrm{ref}}}\subset\mathbb{R}^d$, neighbor count $k$,
ridge/noise floor parameters, and (for \textsc{LR{+}D}) a rank $r$.
\For{$i=1,\dots,N_{\mathrm{ref}}$}
 \State Find $\mc{N}_k(i)$ (the $k$ nearest neighbors of $x_0^i$ in $X$).
 \State Set $h_i^2 \gets \max_{j\in\mc{N}_k(i)} \|x_0^i-x_0^j\|_2^2$.
 \State Set $\bar w_{ij}\gets \exp(-\|x_0^i-x_0^j\|_2^2/(2h_i^2))$ for $j\in\mc{N}_k(i)$.
 \State Normalize $w_{ij}\gets \bar w_{ij}/\sum_{\ell\in\mc{N}_k(i)}\bar w_{i\ell}$.
 \State Compute $\mu_i \gets \sum_{j\in\mc{N}_k(i)} w_{ij}\,x_0^j$.
 \State \textbf{If} mode=\textsc{Diag}, construct $\Sigma_i^{\textsc{Diag}}$ as in \S\ref{app:proxy_diag}.
 \State \textbf{If} mode=\textsc{LR{+}D}, construct $\Sigma_i^{\textsc{LR{+}D}}$ as in \S\ref{app:proxy_lrd}.
 \State Store $(\mu_i,\Sigma_i)$ and the anchor score $\hat s_0(x_0^i)=\Sigma_i^{-1}(\mu_i-x_0^i)$.
\EndFor
\State \textbf{Output:} $\{(\mu_i,\Sigma_i)\}_{i=1}^{N_{\mathrm{ref}}}$ and $\{\hat s_0(x_0^i)\}_{i=1}^{N_{\mathrm{ref}}}$.
\end{algorithmic}
\end{algorithm}

\subsection{Covariance models and hyperparameters}
\label{app:proxy_cov_models}

\subsubsection{Diagonal proxy (\textsc{Diag})}
\label{app:proxy_diag}

For the diagonal proxy, we estimate per coordinate variances from the weighted neighbor
cloud and add an isotropic ridge (noise floor) to stabilize inversion. For $\ell=1,\dots,d$,
we define
\begin{equation*}
  v_{i,\ell}
  \;:=\;
  \sum_{j\in\mc{N}_k(i)} w_{ij}\,\big(x_0^j(\ell)-\mu_i(\ell)\big)^2,
  \qquad
  \tau_i
  \;:=\;
  \gamma\cdot \frac{1}{d}\sum_{\ell=1}^d v_{i,\ell},
\end{equation*}
where $\gamma>0$ is a dimensionless ridge multiplier. We then set
\begin{equation*}
  \Sigma_i^{\textsc{Diag}}
  \;:=\;
  \diag\!\big(v_{i,1}+\tau_i,\dots,v_{i,d}+\tau_i\big).
\end{equation*}
The corresponding proxy score is defined in \cref{eq:proxy_diag}.

\subsubsection{Low rank plus diagonal tail proxy (\textsc{LR{+}D})}
\label{app:proxy_lrd}

For the \textsc{LR{+}D} proxy, we estimate a rank-$r$ principal subspace from the weighted
neighbors and represent the remaining energy by a diagonal tail. We let $M_i\in\mathbb{R}^{k\times d}$
be the weighted residual matrix with rows
\begin{equation*}
  \big(M_i\big)_{(j,\cdot)}
  \;:=\;
  \sqrt{w_{ij}}\,(x_0^j-\mu_i)^\top,
  \qquad j\in\mc{N}_k(i).
\end{equation*}
We compute a rank-$r$ truncated SVD of $M_i^\top M_i$ to obtain $V_i\in\mathbb{R}^{d\times r}$
and $\Lambda_i=\diag(\lambda_{i,1},\dots,\lambda_{i,r})$. We let $\tau_{i,\ell}>0$ be a
per coordinate tail variance (with optional clipping from below to enforce a noise floor).
We then set
\begin{equation*}
  \Sigma_i^{\textsc{LR{+}D}}
  \;:=\;
  V_i\Lambda_i V_i^\top
  \;+\;
  \diag\!\big(\tau_{i,1},\dots,\tau_{i,d}\big).
\end{equation*}
The corresponding proxy score is defined in \cref{eq:proxy_lrd}.

\subsection*{Woodbury inversion}
For implementation, we write $D_i:=\diag(\tau_{i,1},\dots,\tau_{i,d})$ and apply the inverse
using Woodbury to avoid forming dense $d\times d$ matrices:
\begin{equation*}
  \big(D_i+V_i\Lambda_i V_i^\top\big)^{-1}
  \;=\;
  D_i^{-1}
  \;-\;
  D_i^{-1}V_i
  \Big(\Lambda_i^{-1}+V_i^\top D_i^{-1}V_i\Big)^{-1}
  V_i^\top D_i^{-1}.
\end{equation*}

\subsection{$k$-mix recomputation at query points}
\label{app:proxy_recompute}

A single local Gaussian can be biased in regions of high curvature or near crossings.
To reduce this bias, we optionally recompute the proxy score at a query point $x$ by
treating the neighborhood as a compact Gaussian mixture.

We select indices $\{i_m\}_{m=1}^M$ as the $k_{\mathrm{mix}}$ nearest anchors to $x$,
where $M:=k_{\mathrm{mix}}\ll N_{\mathrm{ref}}$. Using the stored anchor parameters
$\{(\mu_i,\Sigma_i)\}_{i=1}^{N_{\mathrm{ref}}}$, we form
\begin{equation*}
  q(x)
  \;:=\;
  \sum_{m=1}^M \pi_m\,
  \mathcal{N}\!\big(x\mid \mu_{i_m},\Sigma_{i_m}\big),
\end{equation*}
where $\pi_m$ are simple priors (for example, proximity weights normalized to sum to one).
The mixture score is
\begin{equation}
  \nabla_x\log q(x)
  \;=\;
  \sum_{m=1}^M \tilde w_m(x)\,\Sigma_{i_m}^{-1}\big(\mu_{i_m}-x\big),
  \qquad
  \tilde w_m(x)
  \;:=\;
  \frac{\pi_m\,\mathcal{N}(x\mid\mu_{i_m},\Sigma_{i_m})}
       {\sum_{j=1}^M \pi_j\,\mathcal{N}(x\mid\mu_{i_j},\Sigma_{i_j})}.
  \label{eq:mog_score}
\end{equation}
We evaluate $\tilde w_m(x)$ using a log--sum--exp computation for numerical stability.

\begin{algorithm}[t]
\caption{Recompute ($k$-mix) mixture score at query $x$}
\label{alg:recompute_kmix}
\begin{algorithmic}[1]
\State \textbf{Input:} query $x$, anchor parameters $\{(\mu_i,\Sigma_i)\}_{i=1}^{N_{\mathrm{ref}}}$, and $k_{\mathrm{mix}}$.
\State Find indices of the $k_{\mathrm{mix}}$ nearest anchors to $x$: $\{i_m\}_{m=1}^M$.
\For{$m=1,\dots,M$}
 \State Compute
 \vspace{-.75em}
 \[
    \ell_m \;\gets\; \log\pi_m
    \;-\;
    \tfrac12 (x-\mu_{i_m})^\top \Sigma_{i_m}^{-1}(x-\mu_{i_m})
    \;-\;
    \tfrac12 \log\det\!\big(2\pi\Sigma_{i_m}\big).
  \]
 \vspace{-1.5em}
\EndFor
\State Let $a\gets\max_m \ell_m$, and set
\vspace{-1.2em}
\[
  \tilde w_m \;\gets\; \exp(\ell_m-a)\Big/\sum_{j=1}^M \exp(\ell_j-a).
\]
\vspace{-1.2em}
\State \textbf{Return} $\hat s_0^{\mathrm{recomp}}(x) \gets \sum_{m=1}^M \tilde w_m\,\Sigma_{i_m}^{-1}(\mu_{i_m}-x)$.
\end{algorithmic}
\end{algorithm}

\noindent\textit{Remark}
The $k$-mix recomputation accepts either \textsc{Diag} or \textsc{LR{+}D} anchors (see \cref{sec:learned_proxy}).
Even with diagonal anchors, recomputation mitigates single Gaussian bias in
high curvature regions, while remaining $O(k_{\mathrm{mix}}d)$ per query.

\subsection{Computational complexity}
\label{app:proxy_complexity}

The costs separate into an offline anchor fit phase and an online query phase.

\subsection*{Neighbor search}
If one computes all $k$NN sets $\{\mc{N}_k(i)\}$ by brute force, the cost is
$O(N_{\mathrm{ref}}^2 d)$ time and $O(N_{\mathrm{ref}}d)$ storage for the data.
In low ambient dimension, tree based methods can reduce this cost, and in higher dimension
approximate $k$NN can be used. Since preprocessing is independent of diffusion time, it is
amortized across all subsequent score evaluations.

\subsection*{Per anchor fitting}
For \textsc{Diag}, computing local moments costs $O(kd)$ time and $O(d)$ memory per anchor,
and applying $(\Sigma_i^{\textsc{Diag}})^{-1}$ is elementwise.
For \textsc{LR{+}D}, estimating the rank-$r$ subspace costs $O(kdr)$ time (or
$O(kd\min\{d,k\})$ with a dense SVD), and storing $V_i$ costs $O(dr)$ memory per anchor.

\subsection*{Query time evaluation}
If one uses the anchor only proxy at an anchor location, no additional cost is incurred
beyond applying $\Sigma_i^{-1}$.
For recomputation at a general query $x$, the cost is dominated by evaluating
$M=k_{\mathrm{mix}}$ components and normalizing mixture weights.
With diagonal anchors this step is $O(k_{\mathrm{mix}}d)$ time.
With \textsc{LR{+}D} anchors, applying the Woodbury inverse yields an effective cost
$O(k_{\mathrm{mix}}dr)$ when $r\ll d$, plus $O(k_{\mathrm{mix}}d)$ for diagonal parts.

\subsection{Asymptotic remarks}
\label{app:proxy_asymptotics}

Under standard smoothness and positivity assumptions and classical $k$NN bandwidth scaling \cite{wasserman2006all}
($k\to\infty$ and $k/N_{\mathrm{ref}}\to 0$), the single component local Gaussian proxy is a
consistent estimator of $s_0(x)$. Standard nonparametric analysis
 \cite{silverman1986density,mack1979multivariate,wasserman2006all} yields
\[
  \Big(\mathbb{E}_{p_0}\,\|\hat s_0(x)-s_0(x)\|_2^2\Big)^{1/2}
  \;=\;
  \mathcal{O}\!\big(N_{\mathrm{ref}}^{-\frac{2}{d+4}}\big)
  \qquad\text{for}\qquad
  k \asymp N_{\mathrm{ref}}^{\frac{4}{d+4}},\footnote{We use $a_N \asymp b_N$ to denote that
  $a_N$ and $b_N$ are of the same asymptotic order.}
\]
up to curvature dependent constants and the chosen covariance structure.
The \textsc{LR{+}D} choice (\cref{sec:learned_proxy}) reduces bias in anisotropic neighborhoods, and the $k$-mix
recomputation further mitigates single mode bias in regions where mixture components
have non negligible overlap by recovering the mixture score \cref{eq:mog_score}.
Because $\|\hat s_0-s_0\|_2\to 0$ as $N_{\mathrm{ref}}\to\infty$ (provided $k/N_{\mathrm{ref}}\to 0$
 \cite{mack1979multivariate}), the TSI term built from $\hat s_0$ remains consistent at small
diffusion times, and the blended score estimator $\sBLND$ \cref{eq:snis_blended_score} inherits the ground truth score behavior in the
limit $N_{\mathrm{ref}}\to\infty$.

\section{Parametric Distillation via a Critic and Gate Network}
\label{app:critic-gate}

In the main text, we developed a nonparametric, variance, optimal blended score estimator $\sBLND$ \cref{eq:snis_blended_score} \(\hat s_{\textsc{blend}}(y,t)\) that combines the TSI and Tweedie identities. To facilitate deployment without a reference set at test time, we now provide a \emph{parametric distillation} strategy that amortizes this blended score estimator $\sBLND$ \cref{eq:snis_blended_score} into a single neural score model. This Appendix outlines the minimal ingredients: the problem setup, a learning objective derived from a variance decomposition, and the theoretical justification for the training procedure.

\subsection{Setup and Learning Objective}
\label{app:critic-gate:setup}
We begin by defining the forward dynamics. Let \(x_0\sim p_0\), \(\xi\sim\mathcal N(0,I)\), and for the Ornstein-Uhlenbeck process, let us define \(y = e^{-t}x_0 + \sigma_t \xi\) with \(\sigma_t^2 = 1-e^{-2t}\).
We denote the per particle signals defined in ~\cref{sect:tweedie,sect:CSE} as follows:
\[
  a(x_0,t) \;=\; e^{t}\,s_0(x_0),
  \qquad
  b(x_0,y,t) \;=\; -\sigma_t^{-2}\!\left(y-e^{-t}x_0\right).
\]
We introduce a \emph{gate} network \(g(y,t;\psi)\in[0,1]\) which produces a blended per particle signal defined by
\[
  z_g(x_0;y,t)
  \;=\;
  \bigl(1-g(y,t;\psi)\bigr)\,a(x_0,t)
  \;+\;
  g(y,t;\psi)\,b(x_0,y,t).
\]
Additionally, a \emph{critic} network \(q(y,t;\omega)\) is introduced to predict the final score as a function of \((y,t)\) alone. We train the parameters \((\psi,\omega)\) by minimizing the population mean squared error (MSE):
\begin{equation}
\label{eq:critic-pop-loss}
  \mathcal L(\psi,\omega)
  \;=\;
  \mathbb E_{x_0,\xi,t}\!
  \left[
    \left\|\,z_g(x_0;y,t)\;-\;q(y,t;\omega)\,\right\|_2^2
  \right].
\end{equation}

\subsection{Variance Decomposition Analysis}
\label{app:critic-gate:ltv}
To understand the efficacy of this objective, we analyze it by conditioning on a fixed time location \((y,t)\). Let \(\pi(\cdot\mid y,t)\) denote the posterior distribution of \(x_0\) given \((y,t)\). Abbreviating \(z_g=z_g(x_0;y,t)\) and \(q=q(y,t;\omega)\), the law of total variance yields the pointwise decomposition
\begin{equation}
\label{eq:ltv}
  \mathbb E\!\left[\|z_g-q\|_2^2 \,\middle|\, y,t\right]
  \;=\;
  \mathrm{Var}_{\pi}\!\left(z_g\right)
  \;+\;
  \Big{\|}
    \mathbb E_{\pi}\!\left[z_g\right]
    \;-\;
    q
  \Big{\|}_2^2.
\end{equation}
Taking the total expectation over \((y,t)\) reveals that minimizing \cref{eq:critic-pop-loss} enforces two complementary roles simultaneously.

First, regarding the critic, for any fixed gate configuration \(g\), the inner minimum of \cref{eq:ltv} is attained when
\[
q(y,t) \;=\; \mathbb E_{\pi}\!\left[z_g(x_0;y,t)\right].
\]
In other words, the critic learns the \emph{MSE optimal} blended score at \((y,t)\) corresponding to the current gate mixture. 

Second, substituting this optimal \(q\) back into \cref{eq:ltv} leaves the gate with the objective to minimize \(\mathrm{Var}_{\pi}(z_g)\). Consequently, \(g\) is driven to find the \emph{variance-minimizing} blend coefficient at each \((y,t)\), matching the optimal \(\lambda^*\) derived in the nonparametric setting.

\subsection{Relation to the Nonparametric Estimator}
This formulation mirrors the nonparametric approach derived in the main text. If we let \(a=s_{\TSI}\) and \(b=s_{\textsc{twd}}\), the variance of the scalar blend \(z_\lambda=(1-\lambda)a+\lambda b\) is minimized by
\[
\lambda^\ast
=
\dfrac{\mathrm{Var}[a]-\mathrm{Cov}[a,b]}
      {\mathrm{Var}[a]+\mathrm{Var}[b]-2\,\mathrm{Cov}[a,b]},
\]
where the moments are computed under \(\pi(\cdot\mid y,t)\). The nonparametric SNIS plug-in estimator approximates this \(\lambda^\ast(y,t)\) using posterior samples. As shown by the decomposition in \cref{eq:ltv}, the parametric critic, and, gate architecture reproduces this exact population objective: the learned \(g(y,t;\psi)\) amortizes the calculation of \(\lambda^\ast(y,t)\), while \(q(y,t;\omega)\) amortizes the resulting blended score, yielding a direct parametric distillation of the nonparametric rule.

\subsection{Proof of Concept Experiment (48-D GMM)}
\label{app:critic-gate:exp}
As a proof of concept for the functionality of critic-gate score distillation, we evaluate a neural distillation of tweedies identity( baseline Denoising score matching) versus a Critic-Gate distilled score network. We test on a dimension \(d{=}48\) Gaussian mixture with strongly curved, filamentary structure, using
a 15-step reverse OU sampler. The Critic-Gate score network is trained using diagonal covariance proxy scores(\ref{sec:learned_proxy}) learned from data alone. Figure~\ref{fig:critic-gate:qual} shows qualitative projections:
the distilled critic preserves filament geometry more closely than a DSM baseline.
Table~\ref{tab:critic-gate:quant} lists quantitative metrics at 15 steps; our Critic-Gate score distillation outperforms the DSM baselines across all divergence metrics.
\vspace{-1.0em}
\begin{figure}[H]
 \centering
 \includegraphics[width=.85\linewidth]{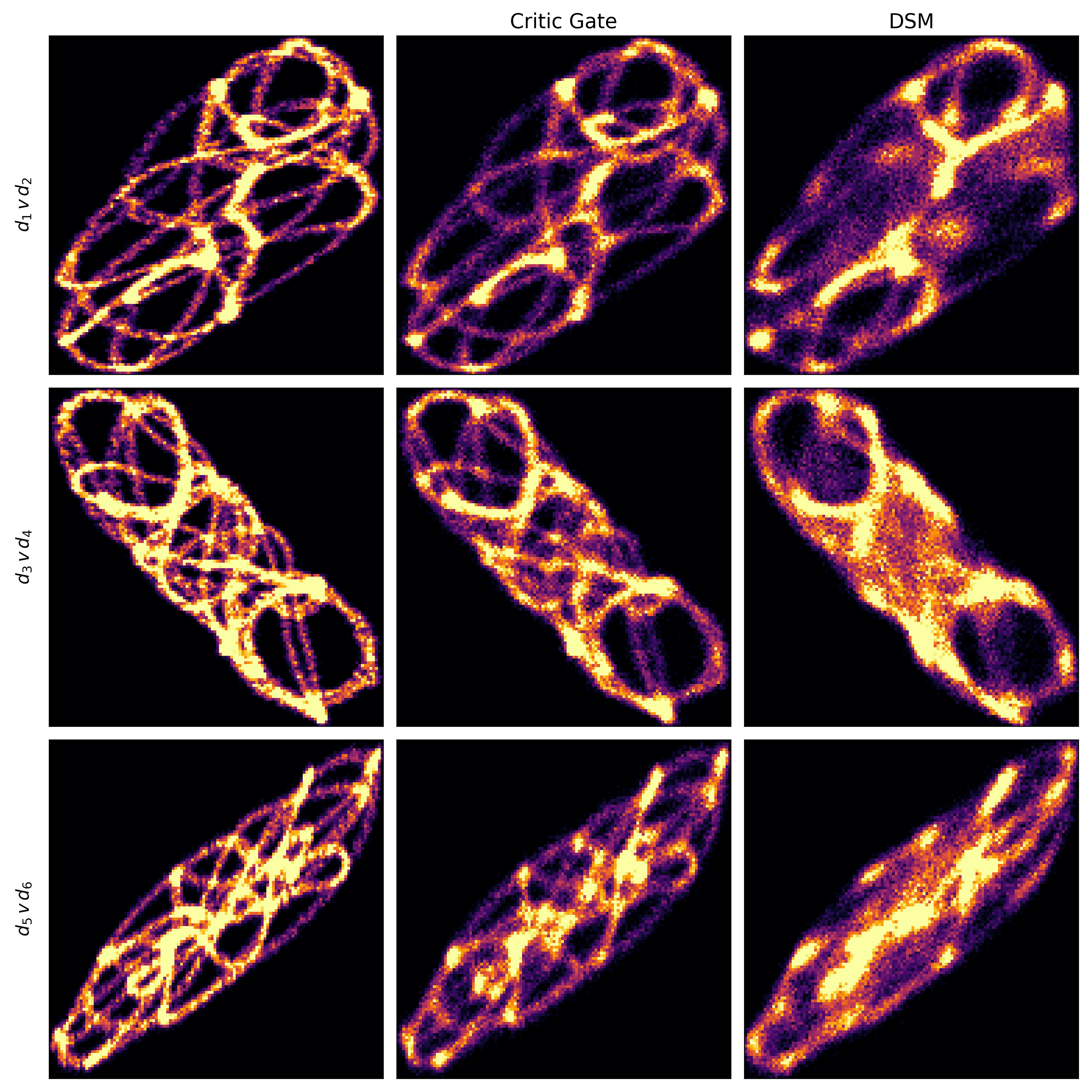}
 \caption{\textbf{Critic, and, Gate distillation on 48-D GMM (15 steps).}
 Qualitative density projections: left column (truth), middle (Critic, Gate), right (DSM).
 The distilled critic, trained by \ref{eq:critic-pop-loss}, recovers thin filamentary setsthat DSM blurs.}
 \label{fig:critic-gate:qual}
\end{figure}
\vspace{-1.0em}
\begin{table}[H]
 \centering
 \caption{\textbf{Critic--Gate v DSM sampling metrics.} 
 Quantitative comparison on the 48-D GMM at 15 sampling steps. 
 The proposed Critic--Gate strategy achieves superior performance across all divergence metrics (MMD, and KSD) compared to standard Denoising Score Matching (DSM).}
 \label{tab:critic-gate:quant}
 \vspace{0.25em}
 \begin{tabular}{lccc}
 \toprule
 Metric & DSM & Critic--Gate (ours) & Floor \\
 \midrule
 MMD@15 $\downarrow$ & $0.03732$ & $0.02507$ & $0.02053$ \\
 KSD@15 $\downarrow$ & $472.7$ & $104.4$ & $15.90$ \\
 \bottomrule
 \end{tabular}
\end{table}

\section{Supplementary Results}
\label{sect:suppres}

This Appendix collects additional plots and experiments deferred from the main text. 
Unless stated otherwise, we use the same reverse time discretization, diagnostic metrics, and evaluation protocol as in the main numerical section; full experimental details are deferred to \cref{sec:repro}.

\subsection{Correlation across time and variance/bias profiles}
\label{app:suppres:cor-var}

Our theory predicts that the Monte Carlo errors of the TSI and Tweedie estimators are negatively correlated (cf.~\cref{sect:negativeCorrelation}). We verify this empirically on the \emph{9D Helix GMM} in \cref{fig:corr_vs_t} by plotting the correlation of the estimator errors as a function of diffusion time $t$:
\[
\varepsilon_T(y,t):=\hat s_{\mathrm{TWD}}(y,t)-s(y,t),\qquad
\varepsilon_{\TSI}(y,t):=\hat s_{\mathrm{TSI}}(y,t)-s(y,t),
\]
\vspace{-1.0em}
\begin{equation*}
\rho(t)=
\frac{\mathbb{E}_{y\sim p_t}\!\big[\langle \varepsilon_T(y,t),\,\varepsilon_{\TSI}(y,t)\rangle\big]}
{\sqrt{\mathbb{E}_{y\sim p_t}\!\|\varepsilon_T(y,t)\|^2}\;\sqrt{\mathbb{E}_{y\sim p_t}\!\|\varepsilon_{\TSI}(y,t)\|^2}}.
\label{eq:correlationCoefficient}
\end{equation*}
We estimate $\rho(t)$ by Monte Carlo over $y\sim p_t$. The \emph{ground truth} curve evaluates the estimators using the exact $s_0$ (and uses the true $s(y,t)$ for error evaluation), while the \emph{proxy} curve replaces $s_0$ by the learned \emph{diagonal} local Gaussian proxy $\hat s_0$ from \cref{sec:learned_proxy} (still comparing to the true $s(y,t)$). We drop time points with low importance sampling quality using the ESS filter from \cref{fn:ess_threshold}. As shown in \cref{fig:corr_vs_t}, $\rho(t)$ is distinctly negative over a broad range of $t$, with a pronounced minimum near $t\approx 10^{-3}$, consistent with the small time anticorrelation predicted in \cref{sect:negativeCorrelation} and sufficient to yield variance cancellation in the blended estimator.
This negative correlation between the estimator errors is preserved when using diagonal proxy score fit to data, albeit weaker especially for larger $t$.
\vspace{-1.0em}
\begin{figure}[H]
 \centering
 \includegraphics[width=.65\linewidth]{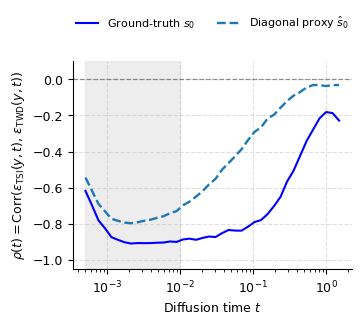}
 \caption{\textbf{Anticorrelation between TSI estimator \cref{eq:nonParametricCSE} and Tweedie estimator \cref{eq:nonParametricTweedie}} on the \textbf{9D Helix GMM}. The grey region highlights the regime where anticorrelation is strongest (near $t\approx 10^{-3}$). The diagonal proxy curve (dashed; \textsc{Diag} from \cref{sec:learned_proxy}) preserves the negative correlation effect that underlies variance reduction in the blended estimator \cref{eq:nonParametricBlend}.}
 \label{fig:corr_vs_t}
\end{figure}

\vspace{-1.0em}
\begin{figure}[H]
\centering
\begin{minipage}{0.48\textwidth}
\centering
\begin{tikzpicture}
\begin{semilogyaxis}[
 width=\textwidth,
 height=0.75\textwidth,
 xlabel={Time $t$},
 ylabel={Variance Time Factor},
 xmin=0, xmax=1.5,
 ymin=0.03, ymax=300,
 grid=both,
 grid style={line width=.1pt, draw=gray!10},
 major grid style={line width=.2pt,draw=gray!30},
 legend pos=north east,
 legend style={fill=white, fill opacity=0.8, draw opacity=1, text opacity=1, font=\small},
 title style={font=\bfseries}
]

\addplot[
 domain=0.01:1.5,
 samples=200,
 color=green!60!black,
 line width=2pt
] {exp(2*x)};
\addlegendentry{TSI: $e^{2t}$}

\addplot[
 domain=0.01:1.5,
 samples=200,
 color=red!70!black,
 line width=2pt
] {exp(-2*x)/(1-exp(-2*x))^2};
\addlegendentry{Tweedie: $\frac{e^{-2t}}{(1-e^{-2t})^2}$}

\addplot[
 dashed,
 line width=1.5pt,
 color=black,
 opacity=0.5
] coordinates {(0.34657, 0.03) (0.34657, 300)};

\node[anchor=west, fill=yellow!30, fill opacity=0.9, text opacity=1, inner sep=3pt, rounded corners, font=\small] 
 at (axis cs:0.38,0.1) {$t^* \approx 0.347$};

\end{semilogyaxis}
\end{tikzpicture}
\end{minipage}
\hfill
\begin{minipage}{0.48\textwidth}
\centering
\begin{tikzpicture}
\begin{semilogyaxis}[
 width=\textwidth,
 height=0.75\textwidth,
 xlabel={Time $t$},
 ylabel={Bias Time Factor},
 xmin=0, xmax=1.5,
 ymin=0.2, ymax=20,
 grid=both,
 grid style={line width=.1pt, draw=gray!10},
 major grid style={line width=.2pt,draw=gray!30},
 legend pos=north east,
 legend style={fill=white, fill opacity=0.8, draw opacity=1, text opacity=1, font=\small},
 title style={font=\bfseries}
]

\addplot[
 domain=0.01:1.5,
 samples=200,
 color=green!60!black,
 line width=2pt
] {exp(x)};
\addlegendentry{TSI: $e^{t}$}

\addplot[
 domain=0.01:1.5,
 samples=200,
 color=red!70!black,
 line width=2pt
] {exp(-x)/(1-exp(-2*x))};
\addlegendentry{Tweedie: $\frac{e^{-t}}{1-e^{-2t}}$}

\addplot[
 dashed,
 line width=1.5pt,
 color=black,
 opacity=0.5
] coordinates {(0.34657, 0.2) (0.34657, 20)};

\node[anchor=west, fill=yellow!30, fill opacity=0.9, text opacity=1, inner sep=3pt, rounded corners, font=\small] 
 at (axis cs:0.38,0.4) {$t^* \approx 0.347$};

\end{semilogyaxis}
\end{tikzpicture}
\end{minipage}
\caption{\textbf{Relative variance and bias (due to SNIS) of the Tweedie and TSI nonparametric score estimators as a function of time $t$}. The former has low variance/bias at large $t$ but diverges at $t=0$, while the latter has low variance/bias at small $t$ but grows exponentially. For both bias and variances, the crossover occurs at the same point as for variance: $t^* = \ln(2)/2 \approx 0.347$.}
\label{fig:variance_profiles}
\end{figure}
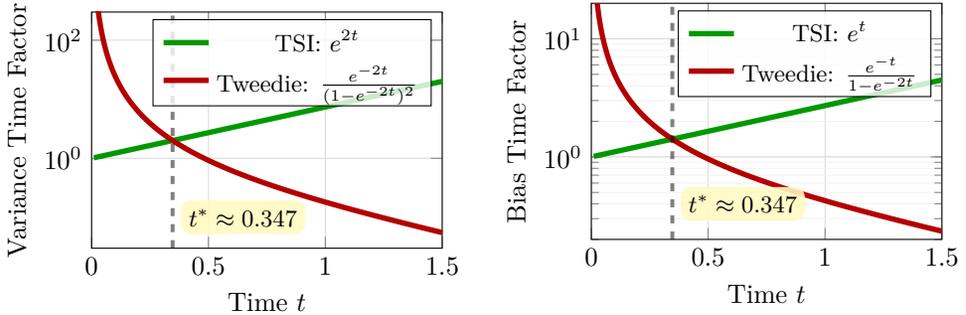
\vspace{-1.0em}

\subsection{Posterior sampling results}
\label{app:suppres:post}

We provide additional posterior sampling diagnostics for synthetic and image inverse problems.

\subsection*{9D Helix GMM with rank-2 likelihood}\label{sec:post-9d}
We consider the $9$D helix Gaussian mixture prior used in the main text, constrained by a rank-$2$ Gaussian likelihood, and visualize samples in PCA planes fitted to posterior reference samples (see \cref{sec:repro} for exact construction).
\cref{fig:post_9d} compares the Tweedie baseline with blended variants.
\vspace{-1.0em}
\begin{figure}[H]
 \centering
 \includegraphics[width=\linewidth]{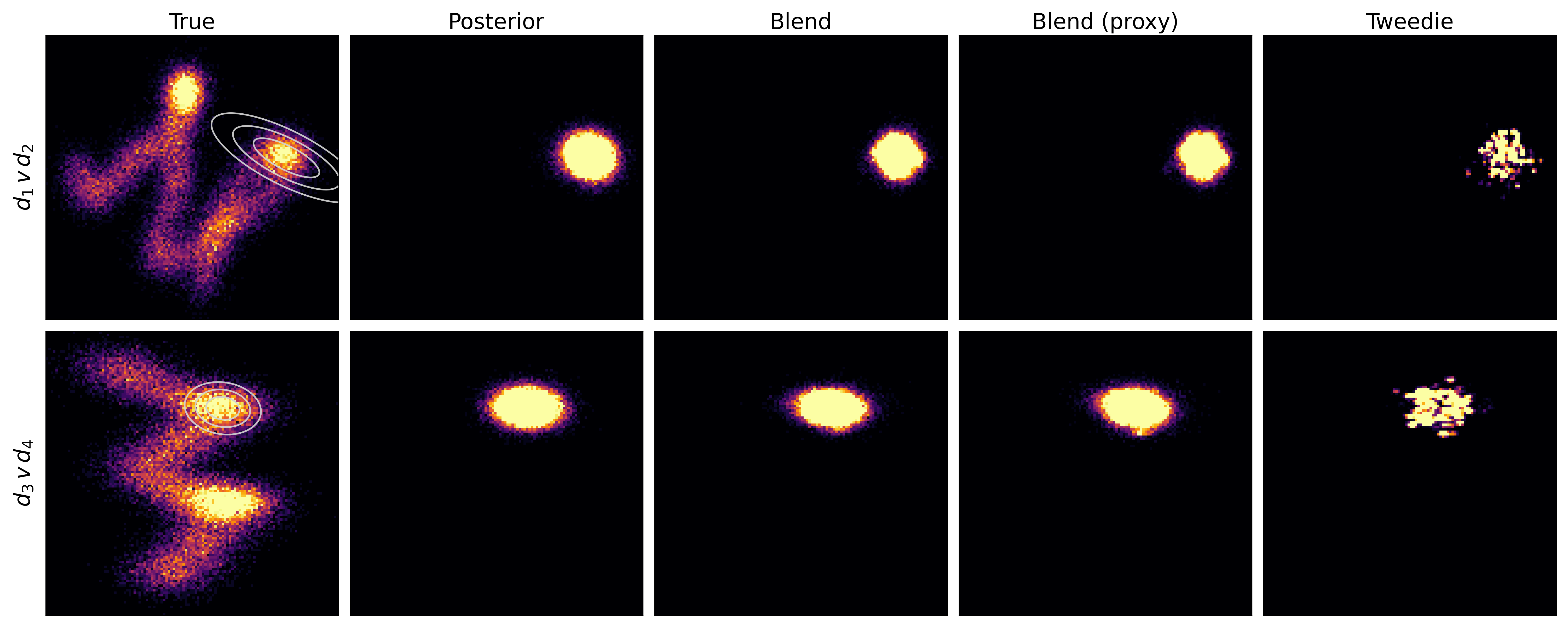}
 \caption{\textbf{Posterior sampling heatmaps on the \textbf{9D Helix GMM} (N=1200).}
 Projected histograms in PCA planes \((d_1,d_2)\) and \((d_3,d_4)\) (principal directions fitted to the \emph{posterior} via importance weighted prior samples).
 \emph{Blend} uses the exact target score; \emph{Blend (proxy)} uses the \textsc{LR{+}D} local Gaussian score proxy from \cref{sec:learned_proxy} fit to the raw data; \emph{Tweedie} is the baseline.
 White contours indicate likelihood level sets.
 Both blends capture the localized posterior manifold, while Tweedie yields fragmented samples because, at small diffusion times, its SNIS estimate becomes dominated by pulls toward a small set of high-weight reference particles (sample memorization), rather than providing a smoothly interpolated local geometric field.
}
 \label{fig:post_9d}
\end{figure}
\vspace{-1.0em}
\noindent In this example, \emph{Blend} and \emph{Blend (proxy)} both approximate the posterior ridge and spread well, while Tweedie tends to fragment/collapse onto a small set of high weight reference particles near the posterior concentration.

\section*{MNIST deblurring panel}\label{subsec:mnist}
We include multi panel posterior sample summaries in \cref{fig:mnist_panel}. These visualize individual posterior samples from the MNIST deblurring problem in \label{subsec:mnist} for all relevant samplers. 

\begin{figure}[H]
 \centering
 \includegraphics[width=.85\linewidth]{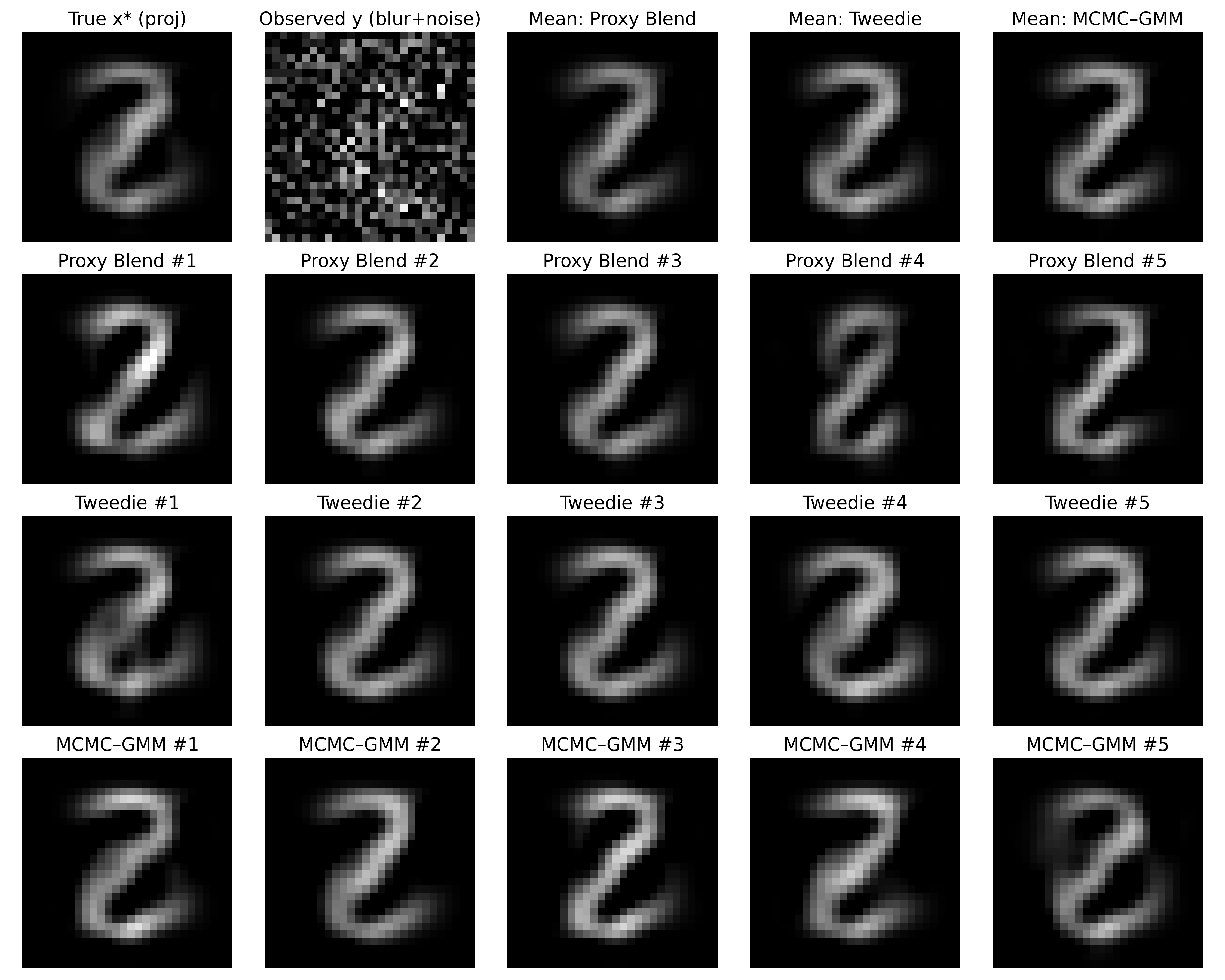}
 \caption{\textbf{MNIST deblurring sample panels.}
 Visual comparison of random posterior samples for MNIST deblurring ($N{=}10{,}000$ references).
 The \textsc{Blend} (Proxy) estimator generates diverse samples that explore the variations in handwriting style allowed by the posterior.
 Pure Tweedie samples exhibit less diversity, tending to revert toward the posterior mean plus visual artifacts.}
 \label{fig:mnist_panel}
\end{figure}
\vspace{-1.0em}
\noindent All samplers produce visually reasonable posterior samples in this run, with \emph{Blend (proxy)} showing the fewest visible artifacts among the displayed panels.

\subsection*{Inverse heat equation}
\label{app:suppres:heat}\label{subsec:heat_inverse}

We report a linear PDE inverse problem in the same posterior sampling framework used for Navier--Stokes (\cref{subsec:ns_inverse}), with identical diagnostics and reverse time integration; only the forward model differs.
Specifically, we infer a log conductivity field $u(x)$ from sparse point observations of the temperature field $\omega(x)$ on $\Omega=(0,1)^2$:
\begin{equation*}
-\nabla\cdot\big(e^{u(x)}\nabla \omega(x)\big)=20,\qquad 
\omega=0\ \text{on }\Gamma_{\mathrm{ext}},\qquad 
n\cdot\big(e^{u(x)}\nabla \omega(x)\big)=0\ \text{on }\Gamma_{\mathrm{root}}.
\end{equation*}
\vspace{-1.2em}
\begin{equation*}
u(x;\alpha)=\sum_{i=1}^{q}\sqrt{\lambda_i}\,\phi_i(x)\,\alpha_i,\qquad \alpha\in\R^{q},\ \ q=15,
\end{equation*}
\vspace{-1.2em}
\begin{equation*}
p_0(\alpha)=\mathcal{N}(0,I),\qquad s_0(\alpha) := \nabla_{\alpha}\log p_0(\alpha) = -\alpha,
\end{equation*}
\vspace{-1.2em}
\begin{equation*}
\mathcal{L}(\alpha) := p(\by\mid \alpha) = \mathcal{N}\!\big(\by;F(\alpha),\sigma_{\mathrm{obs}}^2 I\big),\qquad 
\nabla_{\alpha}\log \mathcal{L}(\alpha) = \frac{1}{\sigma_{\mathrm{obs}}^2}J_F(\alpha)^\top\big(\by - F(\alpha)\big).
\end{equation*}
Representative PCA plane histograms of posterior samples are shown in \cref{fig:pca_histograms}, and reconstructed fields are shown in \cref{fig:field_recon}.
Following the Navier--Stokes table format, we summarize quantitative diagnostics in table  \cref{tab:heat_results}.
\begin{table}[H]
 \centering
 \caption{\textbf{Heat equation quantitative results. } 
 Comparison of posterior sampling accuracy. 
 \textsc{Blend} significantly reduces the kernel score discrepancy (KSD) and distribution mismatch (MMD) compared to the pure Tweedie estimator. 
 The blending mechanism effectively corrects the local geometric errors responsible for Tweedie's poor performance, yielding metrics comparable to the MALA.}
 \label{tab:heat_results}
 \resizebox{\textwidth}{!}{
 \begin{tabular}{lcccccc}
 \toprule
 Method
 & MMD$\to$MALA $\downarrow$
 & $\mathrm{RMSE}_\alpha$ $\downarrow$
 & $\mathrm{RMSE}_{\mathrm{amb}}$ $\downarrow$
 & Fwd Err $\downarrow$
 & KSD $\downarrow$
 & $\widetilde{\mathrm{KL}}$ $\downarrow$ \\
 \midrule
 Tweedie ($\sTWD$)
 & 0.145
 & 0.470
 & 0.111
 & 0.107
 & 16.0
 & 111.2 \\
 Blend Posterior ($\sBLND$)
 & 0.088
 & 0.465
 & 0.106
 & 0.101
 & 1.008
 & 58.4 \\
 MALA (Reference)
 & 0.000
 & 0.457
 & 0.101
 & 0.096
 & 0.883
 & 43.1 \\
 \bottomrule
 \end{tabular}
 }
\end{table}
\vspace{-.5em}
\noindent Visually, \emph{Blend Posterior} matches the MALA reference more closely in the PCA marginals, and this improved agreement is reflected consistently across the reported diagnostics and in field reconstruction plot ; Tweedie exhibits the same kind of fragmentation/collapse seen in the synthetic posterior histograms.
\vspace{-.5em}
\begin{figure}[H]
 \centering
 \includegraphics[width=.8\linewidth]{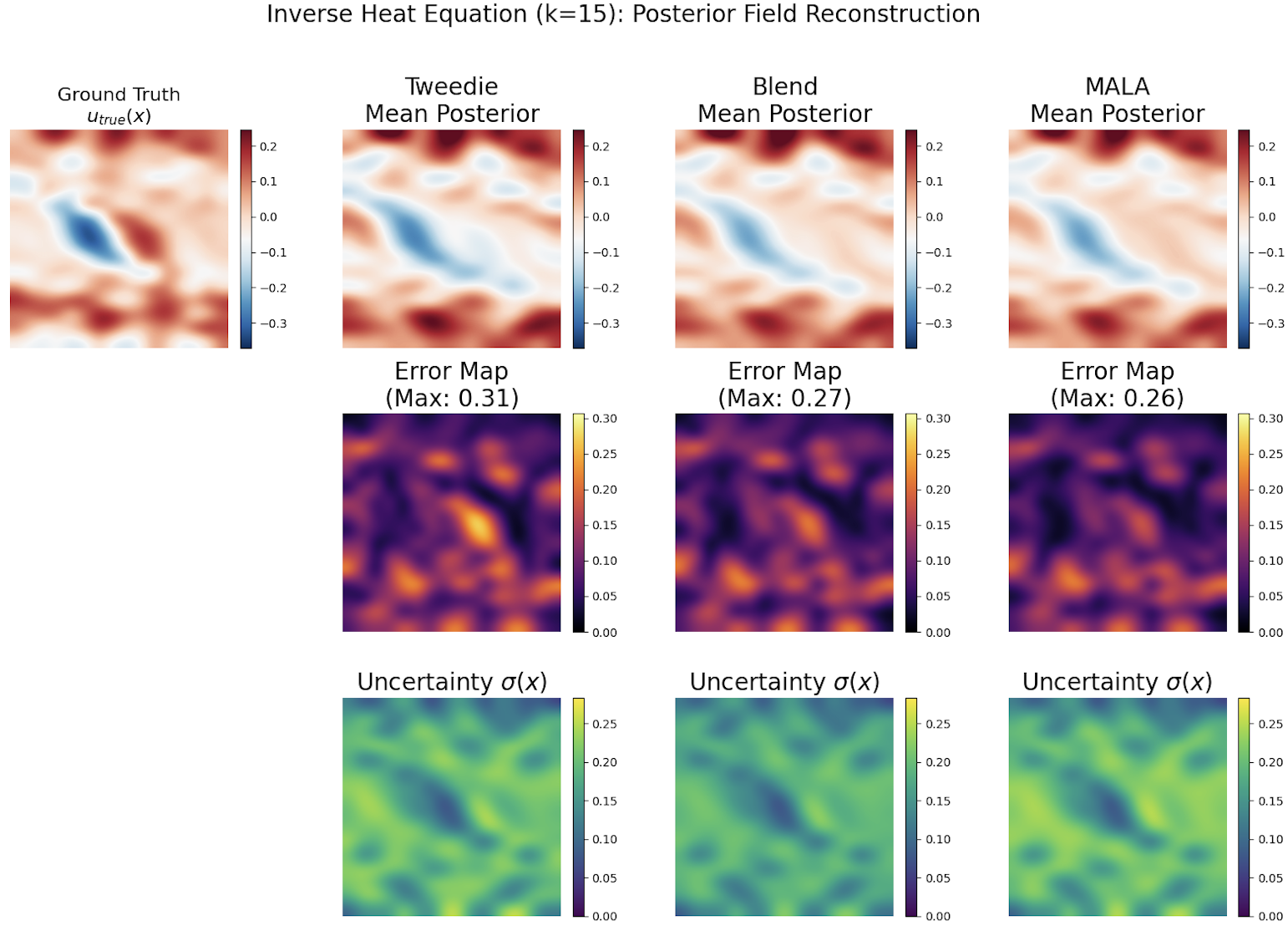}
 \vspace{-1.0em}
 \caption{\textbf{Heat equation conductivity field reconstructions.} 
 Posterior mean reconstruction of the conductivity field $u(x)$. 
 The \textsc{Blend} estimator recovers the spatial structure and intensity of the conductivity anomalies found in the MALA ground truth. 
 In contrast, the Tweedie reconstruction is over-regularized and diffuse, failing to capture the sharp features resolved by the blended score.}
 \label{fig:field_recon}
\end{figure}

\vspace{-1.0em}
\begin{figure}[H]
 \centering
 \includegraphics[width=.75\linewidth]{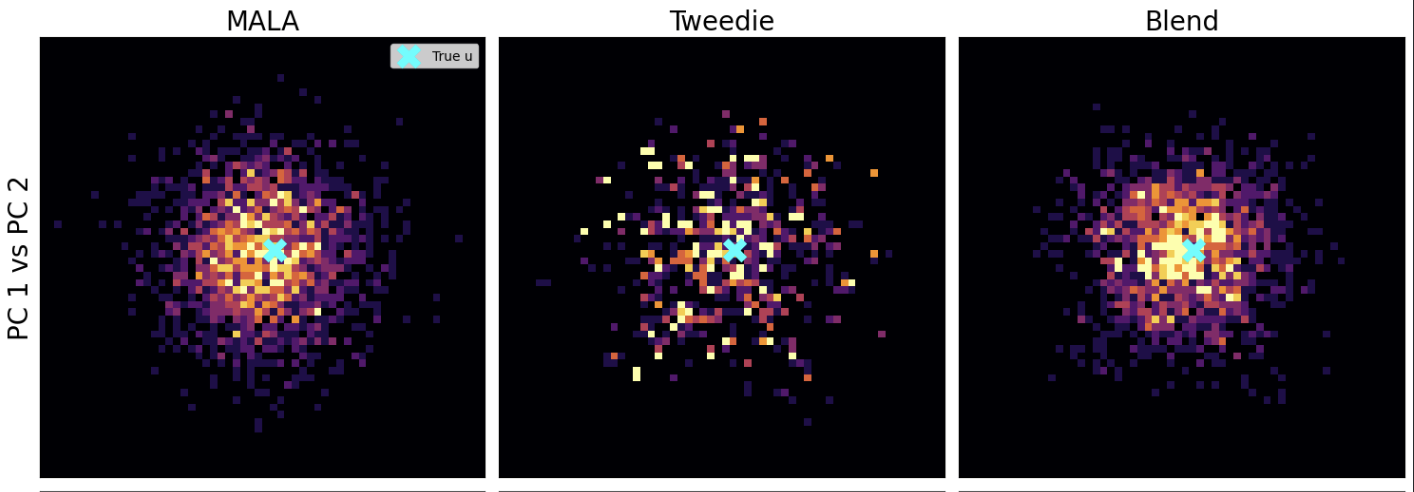}
 \vspace{-.5em}
 \caption{\textbf{Heatmap of Heat equation coefficient posterior histograms.} 
 2D marginal histograms of posterior samples projected onto the leading principal components (Top: PC1 vs PC2; Bottom: PC3 vs PC4). Columns compare MALA (left), Tweedie (center), and Blend (right). 
 The Tweedie estimator exhibits characteristic fragmentation, collapsing mass onto isolated reference samples. 
 \textsc{Blend} restores the continuous posterior geometry, filling the gaps to match the smooth density profiles and ground truth alignment (cyan cross) of the MALA reference.}
 \label{fig:pca_histograms}
\end{figure}

\pagebreak

\section{Reproducibility}
\label{sec:repro}
This Appendix summarizes implementation details and provides a checklist of all experiment
hyperparameters needed to reproduce the figures and tables in the main text.

\subsection{OU corruption, time grid, and reverse time integration}
\section*{Forward OU process}\label{def:ou_forward}
We use the Ornstein--Uhlenbeck (OU) forward dynamics
\[
  dX_t = -X_t\,dt + \sqrt{2}\,dW_t,\qquad X_0\sim p_0,
\]
so that
\[
  X_t \mid X_0=x \sim \mathcal{N}\!\left(e^{-t}x,\ \sigma_t^2 I\right),\qquad
  \sigma_t^2 := 1-e^{-2t}.
\]

\section*{Reverse time sampler}\label{def:reverse_time_sampler}
Given a score estimator $\hat s(y,t)\approx \nabla_y \log p_t(y)$, we integrate the reverse time SDE
We always use the same integrator for Tweedie, TSI, and Blend and across all experiments to ensure comparability:
\[
  dY_t = \bigl(Y_t + 2\hat s(Y_t,t)\bigr)\,dt + \sqrt{2}\,d\bar W_t,
\]
from $t_{\max}$ down to $t_{\min}$, initialized at $Y_{t_{\max}}\sim \mathcal{N}(0,I)$.

\section*{Time grid}\label{def:time_grid}
For all experiments we use a log spaced grid $t_K=t_{\max}>\cdots>t_0=t_{\min}$ with
$t_{\min}=5\times 10^{-4}$ and $t_{\max}=1.5$.

\section*{Heun predictor--corrector}\label{def:heun_pc}
To step from $t_{k+1}$ to $t_k$ with $\delta := t_{k+1}-t_k>0$, define the reverse drift
$f(y,t):= y + 2\hat s(y,t)$ and update
\[
\begin{aligned}
  \textbf{Predictor:}\quad &\tilde y_k = y_{k+1} - \delta f(y_{k+1},t_{k+1}) + \sqrt{2\delta}\,z,\\
  \textbf{Corrector:}\quad &y_k = y_{k+1} - \frac{\delta}{2}\Bigl(f(y_{k+1},t_{k+1}) + f(\tilde y_k,t_k)\Bigr) + \sqrt{2\delta}\,z,
\end{aligned}
\]
where $z\sim \mathcal{N}(0,I_d)$ is drawn once per step and shared between predictor and corrector.
We report results as a function of \emph{number of function evaluations} (NFE), counting each call to $\hat s(\cdot,t)$.

\subsection{SNIS details, median of means, and ESS filtering}
\section*{SNIS for conditional expectations}\label{def:snis_condexp}
For any function $\varphi$, conditional expectations under $p_{t|0}(x_0\mid y)$ are estimated from
a reference set $\{x_0^{(i)}\}_{i=1}^{N_{\mathrm{ref}}}\sim p_0$ via self normalized importance sampling:
\[
  \mathbb{E}_{p_{t|0}(\cdot\mid y)}[\varphi(X_0)]
  \;\approx\;
  \sum_{i=1}^{N_{\mathrm{ref}}}\tilde w_i(y,t)\,\varphi\!\bigl(x_0^{(i)}\bigr),
  \qquad
  \tilde w_i = \frac{\exp(\ell_i-a)}{\sum_j \exp(\ell_j-a)},
\]
where $\ell_i$ are log weights (computed in log space for stability) and $a=\max_i \ell_i$ is the log sum exp shift.
For prior sampling, the weights are proportional to the OU transition density $p(y\mid x_0^{(i)})$.
For posterior sampling, we additionally \emph{tilt} the log weights by $\log p(y_{\mathrm{obs}}\mid x_0^{(i)})$.

\section*{Median of means (MoM)}\label{def:mom}
To reduce sensitivity to heavy tailed importance weights, we compute SNIS estimates over
independent batches (each with its own reference sub sample) and aggregate with a median of means rule \cite{lugosi2019mean}.
Record the number of batches $B$ and the per batch reference size $N_{\mathrm{ref}}^{\mathrm{(batch)}}$
used in each experiment.

\section*{ESS thresholding}\label{def:ess}
We quantify importance weight quality using the effective sample size
\[
  \mathrm{ESS}(y,t) := \frac{1}{\sum_{i=1}^{N_{\mathrm{ref}}}\tilde w_i(y,t)^2}.
\]
Following \cref{fn:ess_threshold} in the main text, time points with $\mathrm{ESS}(y,t) < \tau_{\mathrm{ESS}}$ are discarded, and
we use $\tau_{\mathrm{ESS}} = 0.05\,N_{\mathrm{ref}}$ in all experiments.

\subsection{Local score proxies for $s_0$}
When $s_0(x)=\nabla_x\log p_0(x)$ is unavailable, we approximate it from the reference set using the local Gaussian proxies
from Section~3.5:
\begin{enumerate}
 \item For each anchor $x_0^{(i)}$, compute its $k$ nearest neighbors $N_k(i)$ in the reference set.
 \item Compute a locally weighted mean $\mu_i$ and a structured covariance estimate $\Sigma_i$.
 \item Define the anchor score proxy $\hat s_0(x_0^{(i)}) := \Sigma_i^{-1}(\mu_i - x_0^{(i)})$.
\end{enumerate}
We use two structured families for $\Sigma_i$:
\begin{itemize}
 \item \textbf{Diag:} $\Sigma_i$ is diagonal (per coordinate variance with ridge stabilization).
 \item \textbf{LR{+}D:} $\Sigma_i = V_i V_i^\top + D_i$ with rank $r$ and diagonal $D_i$ (Woodbury inverse at query time).
\end{itemize}

\section*{Parameters to log (proxy)}\label{def:proxy_params_to_log}
Record: $k$ (neighbor count), $r$ (rank for LR{+}D), the ridge/diagonal floor used for stabilization,
and whether the proxy is evaluated at anchors only or recomputed at general queries using a $k_{\mathrm{mix}}$-mixture.


\section*{MALA sampler}\label{def:mala_sampler}
Given a differentiable target density $\pi(\alpha)$ on $\mathbb{R}^q$ (e.g., $\pi(\alpha)\propto p_0(\alpha)\,\mathcal{L}(\alpha)$ in
white box experiments, or an approximate $\pi$ obtained by replacing $p_0$ with a differentiable surrogate prior),
the Metropolis adjusted Langevin algorithm (MALA) uses the proposal
\[
\alpha' \;=\; \alpha \;+\; \frac{h}{2}\,\nabla_\alpha \log \pi(\alpha)\;+\;\sqrt{h}\,z,
\qquad z\sim\mathcal{N}(0,I_q),
\]
where $h>0$ is the stepsize. The proposal density is Gaussian,
\[
q(\alpha'\mid \alpha) = \mathcal{N}\!\Big(\alpha';\ \alpha+\tfrac{h}{2}\nabla\log\pi(\alpha),\ hI_q\Big).
\]
We accept $\alpha'$ with probability
\[
a(\alpha,\alpha') \;=\;
\min\!\left\{1,\ \frac{\pi(\alpha')\,q(\alpha\mid \alpha')}{\pi(\alpha)\,q(\alpha'\mid \alpha)}\right\},
\]
and otherwise retain the current state. In our experiments we run MALA chains for $2{,}000$ iterations with a burn in of
$500$ steps; the retained post burn in states are treated as samples from the reference posterior.

\subsection{Metrics and evaluation protocols}

We provide exact definitions for metrics used in the body of the text.
\section*{MMD}\label{def:mmd}
Given samples 
$X=\{x_i\}_{i=1}^n\sim P$ and $Y=\{y_j\}_{j=1}^m\sim Q$, the (biased) squared MMD is
\[
\mathrm{MMD}^2_k(P,Q) \approx
\frac{1}{n^2}\sum_{i,i'} k(x_i,x_{i'})
+
\frac{1}{m^2}\sum_{j,j'} k(y_j,y_{j'})
-
\frac{2}{nm}\sum_{i,j} k(x_i,y_j).
\]
We use RBF kernels $k_\sigma(x,y)=\exp(-\|x-y\|^2/2\sigma^2)$ with a multiscale bandwidth grid
$\{\sigma_\ell\}_\ell$ obtained from the median heuristic (record multipliers and the subsample size used to estimate the median distance).

\section*{KSD}\label{def:ksd}
For a target score $s(x)=\nabla_x\log\pi(x)$ and a positive definite kernel $k$,
the squared Kernel Stein Discrepancy is $\mathrm{KSD}^2(Q,\pi)=\mathbb{E}_{x,x'\sim Q}[u_s(x,x')]$,
with the Stein kernel
\begin{equation*}
\begin{split}
    u_s(x,x') &= s(x)^\top k(x,x') s(x') + s(x)^\top \nabla_{x'}k(x,x') \\
    &\quad + s(x')^\top \nabla_x k(x,x') + \mathrm{tr}\bigl(\nabla_x\nabla_{x'}k(x,x')\bigr).
\end{split}
\end{equation*}
We use an inverse multiquadric kernel $k(x,y)=(c^2+\|x-y\|^2)^\beta$ with fixed $(c,\beta)$; record $(c,\beta)$ and
whether an unbiased U statistic or V statistic estimator is used.

\section*{Score RMSE}\label{def:score_rmse}
When ground truth scores are available, we report
\[
  \mathrm{RMSE}(\hat s) :=
  \left(\frac{1}{|\mathcal{T}|}\sum_{t\in\mathcal{T}}\mathbb{E}_{y\sim p_t}\bigl[\|\hat s(y,t)-s(y,t)\|^2\bigr]\right)^{1/2},
\]
estimated by Monte Carlo over $y\sim p_t$ on the same $t$-grid.

\section*{MNIST image metrics}\label{def:psnr}
Pixels are scaled to $[0,1]$. For a reconstruction $\hat x$ and ground truth $x$, we report
\[
\mathrm{PSNR}(\hat x,x)=20\log_{10}\!\left(\frac{1}{\sqrt{\mathrm{MSE}(\hat x,x)}}\right).
\]

\section*{Coverage}\label{def:coverage}
Coverage is the fraction of pixels whose ground truth value lies inside the empirical
$90\%$ credible interval computed from posterior samples:
\[
\mathrm{Coverage}(x)\;=\;\frac{1}{d}\sum_{i=1}^d\mathbf{1}\!\left\{x_i\in [Q_{0.05}(\{x_i^{(s)}\}),\,Q_{0.95}(\{x_i^{(s)}\})]\right\}.
\]
We additionally report $\mathbb{E}_{x\sim \pi}[\log p_{\mathrm{KDE}}(x)]$ where $p_{\mathrm{KDE}}$ is a KDE fit to posterior samples;
record the kernel family and bandwidth rule (including any scalar multipliers).


\section*{Coefficient space mean error ($\mathrm{RMSE}_\alpha$)}\label{def:rmse_coeff}
Let $\alpha_\star\in\mathbb{R}^q$ denote the ground truth reduced coordinates used to generate the synthetic observation
(e.g., KL coefficients for Navier--Stokes, PCA coefficients for MNIST). Given posterior samples
$\{\alpha^{(s)}\}_{s=1}^S$, define the posterior mean estimator
\[
\bar\alpha := \frac{1}{S}\sum_{s=1}^S \alpha^{(s)}.
\]
We report the coefficient space mean error
\[
\mathrm{RMSE}_\alpha \;:=\; \frac{\|\bar\alpha-\alpha_\star\|_2}{\sqrt{q}}.
\]

\section*{Ambient space mean error ($\mathrm{RMSE}_{\mathrm{amb}}$)}\label{def:rmse_ambient}
Let $G:\mathbb{R}^q\to\mathbb{R}^d$ denote the deterministic map from reduced coordinates to the ambient object
(field/image), e.g.
\begin{align*}
    \text{Navier--Stokes: } & G(\alpha)=w_0(\cdot;\alpha)\ \text{(discretized on the simulation grid)}, \\
    \text{MNIST: } & G(\alpha)=\mu+U\alpha\in\mathbb{R}^{784}.
\end{align*}
Define the ambient posterior mean estimator
\[
\bar x := \frac{1}{S}\sum_{s=1}^S G(\alpha^{(s)}),\qquad x_\star := G(\alpha_\star).
\]
We report
\[
\mathrm{RMSE}_{\mathrm{amb}} \;:=\; \frac{\|\bar x-x_\star\|_2}{\sqrt{d}}.
\]

\section*{Forward/data fit error (\textnormal{Fwd Err})}\label{def:fwd_err}
Let $F$ denote the forward operator mapping reduced coordinates to the (noise free) observation space, and define
the noiseless observation
\[
y_{\mathrm{clean}} := F(\alpha_\star).
\]
We report the forward relative error of the posterior mean,
\[
\mathrm{FwdErr} \;:=\; \frac{\|F(\bar\alpha)-y_{\mathrm{clean}}\|_2}{\|y_{\mathrm{clean}}\|_2},
\]
which measures how well the inferred posterior mean reproduces the forward map at the sensor/pixel level.

\section*{KL type diagnostic (\texorpdfstring{$\widetilde{\mathrm{KL}}$}{KL tilde})}\label{def:kl_tilde}
In white box settings where the unnormalized posterior density is available up to a normalizer,
\[
p^{\mathrm{post}}(\alpha\mid \by)\ \propto\ p_0(\alpha)\,\mathcal{L}(\alpha),
\]
we report the unnormalized KL form (KL up to an additive constant shared across methods)
\[
\widetilde{\mathrm{KL}}(q)
\;:=\;
-\widehat H(q)\;-\;\mathbb{E}_{\alpha\sim q}\!\big[\log p_0(\alpha)+\log \mathcal{L}(\alpha)\big],
\]
where $\widehat H(q)$ is an entropy estimator for the sampler distribution $q$ (we use a $k$NN entropy estimator \cite{kozachenko1987sample,kraskov2004estimating}),
and the expectation is estimated by Monte Carlo over the generated samples:
\[
\mathbb{E}_{\alpha\sim q}\!\big[\log p_0(\alpha)+\log \mathcal{L}(\alpha)\big]
\;\approx\;
\frac{1}{S}\sum_{s=1}^S \big(\log p_0(\alpha^{(s)})+\log \mathcal{L}(\alpha^{(s)})\big).
\]
For a fixed posterior (fixed data $\by$), $\widetilde{\mathrm{KL}}(q)$ differs from the true $\mathrm{KL}(q\|p^{\mathrm{post}})$
only by an additive constant $-\log Z(\by)$, and hence is comparable across methods on the same inverse problem instance.

\subsection{Experiment specific hyperparameters}
The main text defers concrete numerical settings (e.g., helix parameterization, kernel bandwidth grids, and SNIS batch sizes) to this appendix.Below are the specific hyperparameters, model configurations, and sampling settings used to generate the results in Sections 4.1--4.6.

\section*{9D Helix GMM \cref{sec:prior-9d}}
\textbf{Model:} The prior is a Spectral GMM with $K_{\text{mix}}=64$ components in dimension $d=9$. Means are drawn on a whitened sphere of radius $R=2.0$, with covariance eigenvalues decaying as $\lambda_i \propto i^{-2}$ ($\alpha=1$, ``Helix'' geometry). The forward operator is diagonal, $A=\mathrm{diag}(i^{-1})$, and observations follow $y = Ax + \varepsilon$ with $\varepsilon \sim \mathcal{N}(0,\sigma_{\text{abs}}^2 I)$.
\textbf{Parameters:} We fix $d=9$ and use relative noise $\sigma_{\text{rel}}$ with absolute scaling $\sigma_{\text{abs}} = \sigma_{\text{rel}} \cdot \sqrt{\mathbb{E}\|Ax\|^2}$ (so noise is comparable across dimensions/operators).
\textbf{Sampling:} Heun Predictor Corrector integrator with $K=30$ steps. Time schedule is log spaced from $t_{\max}=2.5$ to $t_{\min}=3 \times 10^{-4}$. Importance weights use a fixed bank of $N_{\text{ref}}=2{,}000$ samples.
\textbf{Metrics:} MMD uses a Gaussian kernel with $\sigma=0.5\sqrt{d/2}$. KSD uses multiscale bandwidths $\sigma \in \{0.1, \dots, 1.0\}$.

\section*{Regime Sweep \cref{sec:regime_study}}
\textbf{Model:} The prior is a Spectral GMM with $K_{\text{mix}}=64$ components. Means are drawn on a whitened sphere of radius $R=2.0$, with covariance eigenvalues decaying as $\lambda_i \propto i^{-2}$ ($\alpha=1$, ``Helix'' geometry). The forward operator is diagonal, $A = \text{diag}(i^{-1})$.
\textbf{Sweep Parameters:} We sweep dimensions $d \in \{3, 6, 12, 24\}$ and relative noise levels $\sigma_{\text{rel}} \in [0.025, 1.0]$ (24 log spaced steps). Absolute noise is scaled as $\sigma_{\text{abs}} = \sigma_{\text{rel}} \cdot \sqrt{\mathbb{E}\|Ax\|^2}$.
\textbf{Sampling:} Heun Predictor Corrector integrator with $K=30$ steps. Time schedule is log spaced from $t_{\max}=2.5$ to $t_{\min}=3 \times 10^{-4}$. Importance weights use a fixed bank of $N_{\text{ref}}=2,000$ samples.
\textbf{Metrics:} MMD uses a Gaussian kernel with $\sigma=0.5\sqrt{d/2}$. KSD uses multiscale bandwidths $\sigma \in \{0.1, \dots, 1.0\}$.

\section*{Navier--Stokes \cref{subsec:ns_inverse}}
\textbf{Model:} 2D Vorticity Stream formulation on a $32 \times 32$ grid with viscosity $\nu=10^{-3}$. The latent variable is the initial vorticity $\omega_0$, parameterized by a KL expansion with $d=24$ ($\mathcal{N}(0, I_{24})$ prior).
\textbf{Observations:} Measurements are taken at final time $T=10$ at $m=100$ random spatial points. Noise level is $\sigma_{\text{obs}}=0.3$. Likelihood gradients are computed via JAX adjoints.
\textbf{Sampling:} Heun integrator with $K=50$ steps and schedule $t \in [1.0, 10^{-3}]$. The reference bank contains $N_{\text{ref}}=20,000$ samples (batch size $1,000$).

\section*{Heat equation Inverse \cref{sect:suppres} }
\textbf{Model:} The domain is a 2D square discretized on a $15 \times 15$ FEM grid ($256$ nodes). The latent parameter is the log conductivity field, parameterized by a KL expansion with dimension $d=15$.
\textbf{Observations:} We observe the temperature field at $m=25$ randomly selected sensors with additive Gaussian noise $\sigma_{\text{obs}}=0.11$.
\textbf{Sampling:} We use the exact prior score for the conditional score model. The sampler runs for $K=50$ steps on a log spaced schedule. The reference set size is $N_{\text{ref}}=20,000$, processed in batches of size $B=4,096$.

\section*{MNIST Deblurring \cref{subsec:mnist}}
\textbf{Model:} The latent space is defined by a PCA projection ($d=15$) fitted on $N=50,000$ training images The prior is modeled as a Gaussian Mixture Model (GMM) with $K_{\text{mix}}=512$ components, fitted via Expectation Maximization (EM) on the latent training data. The forward model is a Gaussian blur with $\sigma_{\text{blur}}=2.6$ pixels and additive noise $\sigma_{\text{obs}}=0.3$.
\textbf{Sampling:} The proxy score is a Local PCA model with rank $r=12$. The sampler runs for $K=20$ steps from $t_{\max}=2.0$ to $t_{\min}=5 \times 10^{-4}$.
\textbf{Validation:} We use $N_{\text{ref}}=10,000$ reference samples for importance sampling and metrics. The ground truth baseline is a MALA sampler running for $3,000$ steps (after $3,000$ warmup steps) with adaptive step size targeting an acceptance rate of $\approx 0.57$.

\subsection{Code Availability}
The source code, configuration scripts, and data generation utilities used to produce the results in this paper are available in the public GitHub repository:
\begin{center}
\url{https://github.com/alduston/CSE_diff}
\end{center}
The repository contains the exact Jupyter notebooks and Python scripts referenced in this reproducibility checklist, allowing for full replication of the regime sweeps, inverse problem solvers, and deblurring experiments.

\section*{Acknowledgments}
We thank Van Hai Nguyen (\texttt{hainguyen@utexas.edu}) for providing code used in the Navier--Stokes and heat-equation examples and for helpful guidance on adapting these solvers for our experiments.

{\footnotesize
\bibliography{refs, references}
}

\bibliographystyle{siam}

\end{document}